\documentclass[10pt,twocolumn,letterpaper]{article}

\usepackage{cvpr}              

\def\paperID{2188}     
\def\confName{CVPR}
\def\confYear{2026}

\usepackage{times}
\usepackage{amsmath,amssymb,amsthm,mathtools,bm}
\usepackage{microtype}
\usepackage{graphicx}
\usepackage{tikz}
\usepackage{tikz-3dplot}
\usetikzlibrary{arrows.meta,positioning,calc}
\tikzset{>=Latex}
\usepackage{xcolor}
\usepackage{booktabs}
\usepackage{multirow}
\usepackage[table]{xcolor}
\usepackage[normalem]{ulem}
\usepackage{lineno}

\usepackage{amsmath,amsfonts,bm}

\def\eqref#1{equation~\ref{#1}}

\def\1{\bm{1}}

\DeclareMathAlphabet{\mathsfit}{\encodingdefault}{\sfdefault}{m}{sl}
\SetMathAlphabet{\mathsfit}{bold}{\encodingdefault}{\sfdefault}{bx}{n}

\newcommand{\R}{\mathbb{R}}

\DeclareMathOperator*{\argmax}{arg\,max}
\DeclareMathOperator*{\argmin}{arg\,min}

\usepackage{enumitem}
\usepackage{calc}

\newtheorem{theorem}{Theorem}[section]
\newtheorem{lemma}[theorem]{Lemma}
\newtheorem{proposition}[theorem]{Proposition}
\newtheorem{corollary}[theorem]{Corollary}
\newtheorem{definition}[theorem]{Definition}
\newtheorem{remark}[theorem]{Remark}
\usepackage{multicol}
\usepackage{cuted}

\newcommand{\gtext}[1]{\textcolor{gray}{#1}}

\definecolor{cvprblue}{rgb}{0.21,0.49,0.74}
\usepackage[pagebackref,breaklinks,colorlinks,allcolors=cvprblue]{hyperref}
\usepackage{url}

\title{POUR: A Provably Optimal Method for Unlearning Representations \\ via Neural Collapse}

\author{
Anjie Le \quad Can Peng \quad Yuyuan Liu \quad J. Alison Noble \\ 
\small Institute of Biomedical Engineering, University of Oxford, UK  \\
}

\begin{document}
\maketitle

\begin{abstract}

In computer vision, machine unlearning aims to remove the influence of specific visual concepts or training images without retraining from scratch.
Studies show that existing approaches often modify the classifier while leaving internal representations intact, resulting in incomplete forgetting.
In this work, we extend the notion of unlearning to the {representation level}, deriving a three-term interplay between forgetting efficacy, retention fidelity, and class separation.
Building on Neural Collapse theory, we show that the orthogonal projection of a simplex Equiangular Tight Frame (ETF) remains an ETF in a lower dimensional space, yielding a provably optimal forgetting operator.
We further introduce the \textbf{Representation Unlearning Score} (RUS) to quantify representation-level forgetting and retention fidelity.
Building on this, we introduce \textbf{POUR} (\textbf{P}rovably \textbf{O}ptimal \textbf{U}nlearning of \textbf{R}epresentations), a geometric projection method with closed-form (\textbf{POUR-P}) and a feature-level unlearning variant under a distillation scheme (\textbf{POUR-D}).
{Experiments on CIFAR-10/100 and PathMNIST demonstrate that POUR achieves effective unlearning while preserving retained knowledge, outperforming state-of-the-art unlearning methods on both classification-level and representation-level metrics. Code is available at: \url{https://github.com/ale256/representation_unlearning}.}
\vspace{-0.6cm}
\end{abstract}

\section{Introduction}

The ability to selectively remove knowledge from trained models has become an increasingly important requirement for modern machine learning systems. 
Motivations include regulatory compliance with data protection laws (e.g., the ``right to be forgotten'') \cite{pipl2021,gdpr2016,ccpa2018}, mitigating reliance on spurious correlations, and ensuring the safe deployment of large pre-trained models in sensitive domains such as autonomous driving and medical imaging. In these applications, models often need to forget outdated, biased, or privacy-sensitive visual data while retaining general visual understanding for reliable downstream use. 
Figure \ref{fig:gradcam} illustrates this goal on the PathMNIST dataset, showing how our method erases the ``adipose'' class while preserving the remaining categories.

\begin{figure}
    \centering
    \includegraphics[width=0.8\linewidth, height=4cm, trim={1.6cm 4.5cm 2.6cm 1cm}]{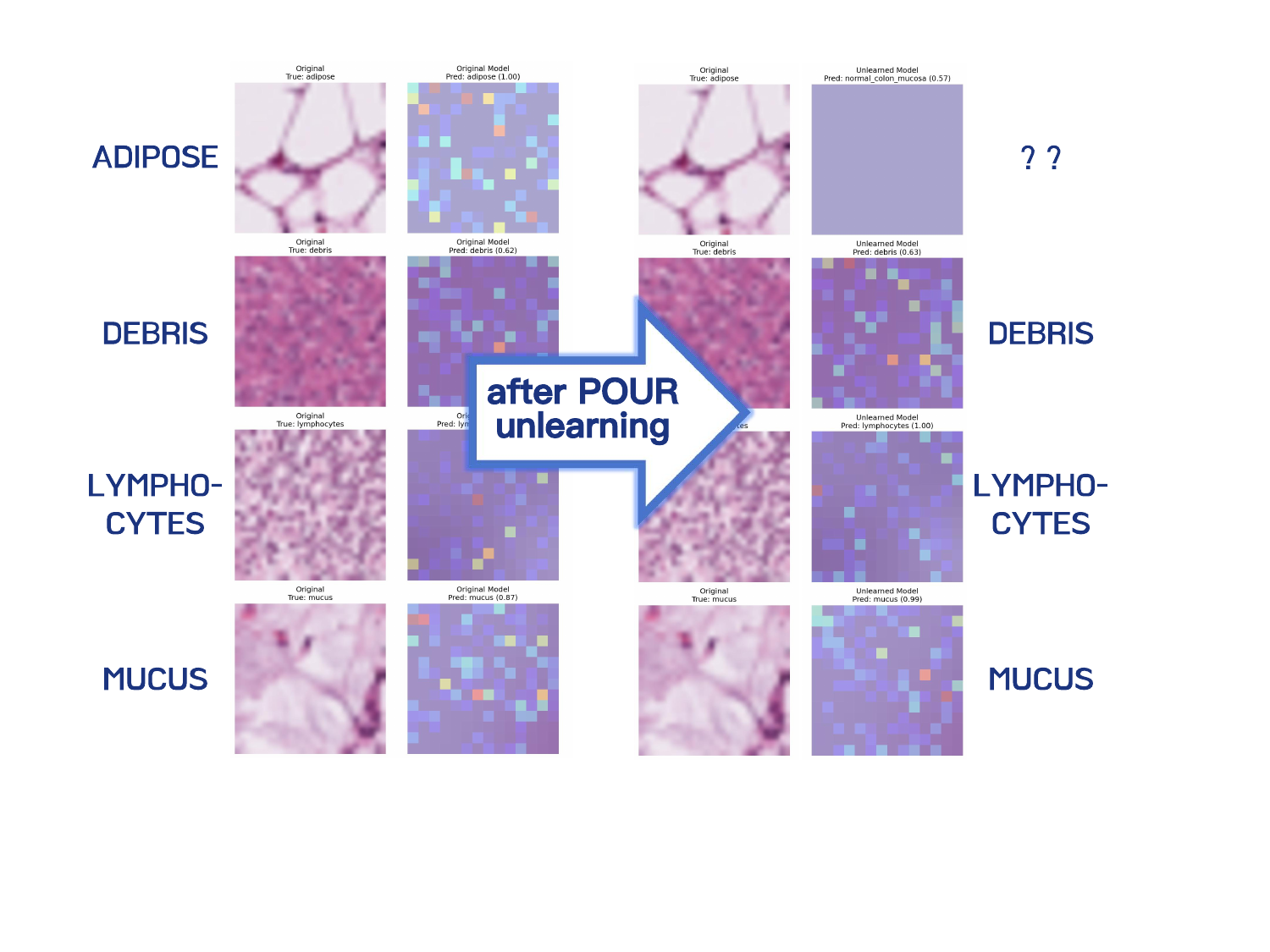}
    \caption{\textbf{Grad-CAM visualization on PathMNIST before and after unlearning.} 
Each row shows a tissue class. After applying POUR on the \textit{adipose} class, its Grad-CAM signal vanishes, 
while the retained classes (\textit{debris}, \textit{lymphocytes}, \textit{mucus}) preserve clear and distinct attention patterns.}
    \label{fig:gradcam}
    \vspace{-0.5cm}
\end{figure}

A growing literature on {machine unlearning} has explored how to make models forget a specific class, subset, or concept without retraining from scratch \citep{bourtoule2021machine,golatkar2020eternal,chundawat2023zeroshot,sepahvand2025selective}. 
Previous work on machine unlearning has primarily focused on aligning the prediction probabilities on the forget and retain sets. 
This line of research, often referred to as {weak unlearning} \citep{xu2023machineunlearningsurvey,golatkar2020forgetting}, aims to ensure that the distributions of the final logits produced by the original and unlearned models are indistinguishable. 
However, recent studies \citep{kim2025arewe} have questioned whether such methods {truly forget} the targeted information, as they often only perturb classifier logits while leaving the underlying feature representations largely unchanged. This shallow modification leaves residual information that can lead to privacy leakage \citep{zhou2025decoupled}. 
This issue is particularly critical for deep vision encoders whose internal representations can still leak forgotten visual concepts through linear probing or feature inversion \cite{kim2025arewe,golatkar2020eternal}.


In parallel, theoretical advances have revealed that deep visual classifiers exhibit highly structured geometric behavior at convergence. 
The theory of {Neural Collapse} (NC) shows that class features concentrate around equidistant centroids and classifier weights align to form a simplex Equiangular Tight Frame (ETF) \citep{papyan2020neuralcollapse}. 
This geometry provides a powerful lens for reasoning about class-level knowledge in image recognition: each class corresponds to a single ETF direction, and as we propose in this work, forgetting a class corresponds to removing its associated vector from the representation space. 
Previous work~\citep{kodge2024deep} has pursued heuristic realizations of ``projection as unlearning'' through Singular Value Decomposition (SVD)-based decomposition in activation space. However, the method lacks geometric consistency and theoretical guaranties.

In this work, we first extend the traditional notion of weak unlearning to the \emph{representation level}, and propose the \textbf{Representation Unlearning Score (RUS)} as a principled feature space metric for quantifying how well a model forgets. 
Building on this formulation, we observe that the restructuring of forget and retain representations occurs at different stages of unlearning, governed by class separation.
We also establish two new properties from the current NC framework: (i) a simplex ETF structure certifies Bayes optimality in balanced classification, and (ii) the orthogonal projection of a simplex ETF remains a simplex ETF. 
Therefore, class forgetting can be implemented as a projection operator that preserves the NC geometry along the direction of the forget classes, which leads to our proposed algorithm \textbf{POUR}
(\textbf{P}rovably \textbf{O}ptimal \textbf{U}nlearning of \textbf{R}epresentations).

POUR comes in two variants: a closed-form projection (\textbf{POUR-P}) that performs instantaneous forgetting, and a projection-guided distillation scheme (\textbf{POUR-D}) that propagates forgetting into the feature extractor using only the forget set through feature alignment.

In summary, our contributions are threefold:
\begin{itemize}
    \item We reformulate machine unlearning at the representation level and introduce RUS based on feature-space discrepancy.
    \item We establish a three-term interplay among forget, retain and class separation for unlearning problems, and derive two new theoretical properties of NC geometry, linking the simplex ETF structure to Bayes optimality and projection invariance.
    \item We propose POUR, a provably optimal projection-based unlearning algorithm with both closed-form and feature-adaptive variants, and formally prove its optimality.
\end{itemize}

{Experiments on CIFAR-10 and CIFAR-100 demonstrate that POUR effectively removes targeted visual concepts while preserving performance on retained classes.
On PathMNIST, POUR further exhibits consistent generalization under domain shift, achieving reliable performance across both internal and external test sets.}
\section{Reformulating Unlearning: Representation Removal with the Forget Set}
\label{sec:theory}

We define the class-centric machine unlearning problem with standard notation as follows.  
Let $\mathcal{D} = \{(\mathbf{x}_i, y_i)\}_{i=1}^n$ denote the entire dataset, where each sample $\mathbf{x}_i \in \mathcal{X} \subseteq \mathbb{R}^d$ is a $d$-dimensional vector, $y_i \in \mathcal{Y} = \{1, 2, \dots, C\}$ is the ground truth label among $C$ classes, and $n$ is the size of $\mathcal{D}$.  
A training algorithm $\mathcal{A}$ maps a dataset $\mathcal{D}$ to a model $M = (\theta, W)$, where $\theta: \mathcal{X} \!\to\! \mathcal{Z}$ is a feature extractor and $W$ is a classifier head.  
For each sample, $M(\mathbf{x}_i)$ approximates its label $y_i$.  

We partition $\mathcal{D}$ into a retain set $\mathcal{D}_r$ and a forget set $\mathcal{D}_f$, such that $\mathcal{D} = \mathcal{D}_f \cup \mathcal{D}_r$ and $\mathcal{D}_f \cap \mathcal{D}_r=\emptyset$.  
The goal of unlearning is to remove the influence of the forget set $\mathcal{D}_f \subset \mathcal{D}$ from the trained model while preserving performance on $\mathcal{D}_r$.  
Let $M_r = \mathcal{A}(\mathcal{D}_r)$ denote the reference model retrained from scratch using only $\mathcal{D}_r$.  
The unlearning process $\mathcal{U}(M, \mathcal{D})$ is then defined as a function that takes a trained model $M = \mathcal{A}(\mathcal{D})$ and produces a new model $M_f$ that behaves similarly to $M_r$.  

In the class-forgetting setting, if we wish to forget a class $u \in \mathcal{Y}$, then $\mathcal{D}_f = \mathcal{D}_u := \{\mathbf{x}_i : y_i = u\}$.  
The retain set $\mathcal{D}_r$ is often inaccessible due to privacy or practical constraints.  
Following \citep{zhou2025decoupled, cha2024learning}, we therefore consider the realistic setting where unlearning is performed using only the forget set $\mathcal{D}_f$, denoted by $\mathcal{U}(M, \mathcal{D}_f)$.

\subsection{Definition}  

Recent findings by \citet{kim2025arewe} demonstrate that focusing solely on final logits does not guarantee complete forgetting, as the forgotten information may still be recoverable through linear probing.
This observation highlights the need to investigate forgetting at the representation level, within the feature extractor itself, rather than only at the output layer.
Motivated by this, \textbf{we propose the concept of representation-level weak unlearning}, which, in contrast to the original definition of weak unlearning, explicitly accounts for the internal feature representations of models. 


\begin{definition}[Representation-Level Weak Unlearning]
\label{def:rep-unlearn}

An unlearning procedure $\mathcal{U}$ applied to $(M,\mathcal{D}_f)$ is said to satisfy 
representation-level weak unlearning if the feature distributions of the unlearned model 
are close to those of the reference model $M_r$, i.e.
\begin{equation}
    \mathcal{K}\!\Big(P_z^{\,\mathcal{U}(M,\mathcal{D}_f)},\; P_z^{\,M_r}\Big) < \epsilon, \label{eq:defn}
\end{equation}
for some distributional discrepancy measure $\mathcal{K}$ (e.g.\ MMD, Wasserstein-2, or Energy Distance) 
and tolerance $\epsilon>0$. 
Here $P_z^M$ denotes the distribution of features $z=\theta(x)$ 
induced by model $M$ on input $x$, where $x\sim \mathcal{D}$.  
\end{definition}

Intuitively, this condition requires that, after unlearning, the feature representations of $\mathcal{D}$ are statistically indistinguishable from those produced by the retrained reference model $M_r$.


\subsection{Practical Estimation of K}

Due to the stochastic nature of training, comparing the feature distributions of the unlearned model and the retrained model is not straightforward.
{We require a representation measure that is robust to randomness, including random initialization, rotations of the feature basis, and uniform rescaling of feature magnitudes. }
Therefore, we adopt {Centered Kernel Alignment (CKA)}~\citep{kornblith2019similarity} as a practical estimator of representation similarity, for its invariance to scaling and rotation; {additional justification is provided in Appendix~\ref{app:cka_robust}.}

\label{sec:cka}
Formally, given two representation matrices $X, Y \in \mathbb{R}^{n \times d}$ from two models evaluated on the same set of $n$ samples, their linear CKA similarity is defined as
\begin{equation}
\text{CKA}(X,Y)
=
\frac{\langle X X^\top,\, Y Y^\top \rangle_F}
{\|X X^\top\|_F \, \|Y Y^\top\|_F},
\end{equation}
where $\langle \cdot, \cdot \rangle_F$ denotes the Frobenius inner product, making it robust to common randomness in neural networks. 

In the unlearning setting, for $M_o$, $M_f$, and $M_r$, the original, unlearned, and retrained models, we define two families of CKA similarities:
\begin{equation}
\resizebox{\hsize}{!}{$
\begin{aligned}
& \text{CKA}^{(o)}_f := \text{CKA}(M_f, M_o; \mathcal{D}_f), 
\quad
\text{CKA}^{(o)}_r := \text{CKA}(M_f, M_o; \mathcal{D}_r), \\
& \text{CKA}^{(r)}_f := \text{CKA}(M_f, M_r; \mathcal{D}_f), 
\quad
\text{CKA}^{(r)}_r := \text{CKA}(M_f, M_r; \mathcal{D}_r).
\end{aligned}
$}
\end{equation}
The superscript $(o)$ indicates comparison with the original model, which is commonly available in practice, while $(r)$ denotes comparison with the retrained model, which approximates the theoretical ideal.

To jointly balance the forgetting and retention objectives, 
we define the {Representation Unlearning Score (RUS)} as
\begin{equation}
\text{RUS}^{(*)}
:=
\frac{2 \,\Phi_f^{(*)}\,\text{CKA}^{(*)}_r}
{\Phi_f^{(*)} + \text{CKA}^{(*)}_r},
\quad 
(*) \in \{(o), (r)\},
\end{equation}
where
\[
\Phi_f^{(o)} = 1 - \text{CKA}^{(o)}_f,
\qquad
\Phi_f^{(r)} = \text{CKA}^{(r)}_f.
\]

$\text{RUS}^{(r)}$ represents the evaluation using the retrained model as the reference, 
while $\text{RUS}^{(o)}$ provides a practical surrogate when the retrained model is inaccessible, for example, due to computational cost or data availability. 
{RUS$^{(*)}$ corresponds to the harmonic mean 
of retention alignment $\text{CKA}_r$ and the forgetting indicator $\Phi_f^{(*)}$,
rewarding methods that achieve both effective forgetting and faithful retention. The definition ensures that both variants of RUS take values in $[0,1]$ and increase with successful forgetting. }

\subsection{Theoretical Characterization}
\label{subsec:theo_charac}
We next show that the discrepancy between unlearned and reference feature distributions can be decomposed into {three interpretable components} that directly capture forgetting efficacy, retention fidelity, and class separation.

\begin{proposition}[Decomposition of $\mathcal{K}$ Bound]
\label{prop:decoupled-k}

Let $P_z^{(f)}$ and $P_z^{(r)}$ denote the feature distributions induced by the unlearned and retrained models, respectively. 
For a forget class $u \in \mathcal{Y}$ and an Integral Probability Metric (IPM) $\mathcal{K}$ defined on the feature space, 
by the law of total probability we can express
\begin{equation}
\resizebox{\hsize}{!}{$
P_z^{(f)} = \alpha\, P^{(f)}_u + (1-\alpha)\, P^{(f)}_{\neg u}, 
\quad
P_z^{(r)} = \beta\, P^{(r)}_u + (1-\beta)\, P^{(r)}_{\neg u},
$}
\end{equation}
where $\alpha := P_z^{(f)}(\hat{y}=u)$ and $\beta := P_z^{(r)}(\hat{y}=u)$ are the predicted probabilities of the unlearning class under each model, 
and $P^{(\cdot)}_{ u},P^{(\cdot)}_{\neg u}$ denote the unlearned and retained class feature distribution.
Then, the discrepancy between the unlearned and retrained feature distributions is bounded as
\begin{equation}
\resizebox{\hsize}{!}{$
\begin{aligned}
&\Big|
  \alpha\, \mathcal{K}\!\big(P^{(f)}_u, P^{(r)}_u\big)
  - (1-\alpha)\, \mathcal{K}\!\big(P^{(f)}_{\neg u}, P^{(r)}_{\neg u}\big)
\Big|
- |\,\alpha-\beta\,|\, \Delta_c
\\[3pt]
&\hspace{2.1cm}\le\;
\mathcal{K}\!\big(P_z^{(f)}, P_z^{(r)}\big)
\\[3pt]
&\hspace{2.1cm}\le\;
\underbrace{|\,\alpha-\beta\,|\, \Delta_c}_{\text{class separation}}
\;+\;
\underbrace{\alpha\, \mathcal{K}\!\big(P^{(f)}_u, P^{(r)}_u\big)}_{\text{forgotten-class discrepancy}}\\
&\hspace{2.5cm}\;+\;
\underbrace{(1-\alpha)\, \mathcal{K}\!\big(P^{(f)}_{\neg u}, P^{(r)}_{\neg u}\big)}_{\text{retained-class discrepancy}}\,,
\end{aligned}
$}
\end{equation}
where $
\Delta_c := \mathcal{K}\!\big(P^{(r)}_u, P^{(r)}_{\neg u}\big).
$
\end{proposition}

A complete statement and proof is given in Appendix~\ref{lem:decoupled-k}.  
When unlearning is performed on the forget set, we have $\beta = 0$, and the forgetting coefficient $\alpha$ decreases gradually from approximately $1$ to $0$ as unlearning proceeds.
In this regime, the effective supervision target becomes $P_{\neg u}^{(r)}$, while both $P_{u}^{(f)}$ and $P_{\neg u}^{(f)}$ are progressively aligned toward the retained-class manifold of the retrained model at different stages of unlearning.
The class-separation term simplifies to $\alpha \Delta_c$, indicating that stronger geometric separation in the retrained feature space enables more effective guidance for forgetting at early stages.
Consequently, unlearning can often be achieved using the forget set alone when class separation is sufficient; however, when separation in the retrained model is weak, the forget-set-only strategy becomes less effective.
This phenomenon is also confirmed empirically, as discussed in Section~\ref{subsec:cifar100}.
\section{Representation Space: Neural Collapse and Simplex ETF Geometry}
\label{subsec:nc-background}

The trade-off derived above reveals that unlearning dynamics are fundamentally governed by the geometry of class separation in representation space.
To analyze this geometry, we draw inspiration from the theory of {Neural Collapse} (NC), which shows that deep classifiers trained with cross-entropy loss organize features into a Simplex ETF during the {Terminal Phase of Training (TPT)}~\citep{papyan2020neuralcollapse} under certain assumptions described in Appendix \ref{app:nc_assumptions}.

Formally, for each class $i$, the learned feature representation takes the form
\[
z_\theta(x) = \alpha(x)\,v_i,
\]
where $z_\theta(x)$ denotes the feature extractor $\theta$ applied to input $x$, $\alpha(x) > 0$ is a class-dependent scaling factor, and $v_i \in \mathbb{R}^d$ is a unit direction representing the class mean.  
The set of class directions $\{v_i\}_{i=1}^C$ lies in a $(C{-}1)$-dimensional subspace and forms a simplex ETF, i.e.,
\[
\|v_i\| = 1, \quad v_i^\top v_j = -\tfrac{1}{C-1} \text{ for } i \neq j, \quad \text{and} \quad \sum_{i=1}^C v_i = 0,
\]
which implies that class means are maximally separated and symmetrically arranged in feature space. Furthermore, the classifier’s weight vectors align with these class directions.

{A detailed formulation of the underlying assumptions and the full NC statements are included in Appendix~\ref{app:nc_assumptions}.} An empirical investigation of the NC phenomenon is given in Section \ref{subsec:nc_pheno}. 
This provides a natural foundation for understanding and manipulating class forgetting at the representation level.

\label{subsec:proj-etf}

In prior work, this ETF geometry has primarily been regarded as a {descriptive limit} of training dynamics. 
In this work, \textbf{we establish two new properties of the NC phenomenon}. 
First, we show that the simplex ETF geometry is not only a consequence of optimization, but also a sufficient condition for Bayes optimality under natural statistical assumptions. 
In this sense, the ETF structure serves as an optimality certificate. 
Second, we demonstrate that the ETF geometry is preserved under orthogonal projection when one vertex is removed, thereby providing the geometric foundation for our proposed unlearning method.

\subsection{ETF as an Optimal Condition}
\label{subsec:class_assump}
In addition to the NC assumptions, we further assume 
that the class-conditional feature distributions are isotropic Gaussians,
\[
x \mid y=i \;\sim\; \mathcal{N}(v_i, \sigma^2 I_d),
\]
where $\|v_i\|=1$ for all $i$ and $\sigma^2>0$ is fixed.
Empirically, this corresponds to features within each class clustering around a well-defined mean with approximately uniform variance in all directions, consistent with an isotropic Gaussian structure. In practice, datasets with sufficiently large sample sizes naturally satisfy this assumption by the Law of Large Numbers.
Under these assumptions, we have the following proposition:

\begin{proposition}[ETF $\Rightarrow$ Bayes optimality]
\label{prop:etf-opt}
\hspace{1mm}
\begin{enumerate}[label=(\roman*), leftmargin=*, labelsep=0.5em, align=parleft]
\item the simplex ETF uniquely maximizes the minimum pairwise angle among class means,
\item it maximizes the multiclass angular margin of the nearest-class-mean classifier, and
\item in the limit $\kappa \to \infty$ or as the within-class variance $\sigma^2 \to 0$, the induced decision rule coincides with the Bayes-optimal classifier.
\end{enumerate}

\end{proposition}

Detailed proof of is provided in Appendix~\ref{app:bayes_ncm}.

\subsection{ETF Stability under Projection.} 
The second property concerns the robustness of ETF geometry under dimensionality reduction. 
Geometrically, removing one vertex of a regular simplex and projecting the remaining vertices 
onto the complementary subspace yields a smaller regular simplex. 
This effect, visualized in Figure \ref{fig:proj_etf} for a $C=4$ case, is captured in the following proposition.

\begin{proposition}[Projection of a simplex ETF remains a simplex ETF]
\label{prop:etf-projection}
With the assumption and notations above, fix $u\in\{1,\dots,C\}$ and let $P=I-v_u v_u^\top$ be the orthogonal projector onto $v_u^\perp$. 
For $i\in\mathcal{Y}_{\neg u}$, define $g_i = \tfrac{P v_i}{\|P v_i\|}$. 
Then $\{g_i\}_{i\in\mathcal{Y}_{\neg u}} \subset v_u^\perp \cong \mathbb{R}^{C-2}$ is a simplex ETF of size $C-1$, i.e.
$
g_i^\top g_j = -\tfrac{1}{C-2} (\forall i\neq j), 
\sum_{i\in\mathcal{Y}_{\neg u}} g_i = 0.
$
\end{proposition}

\begin{figure}

\centering
\tdplotsetmaincoords{70}{110}
\begin{tikzpicture}[tdplot_main_coords, scale=2.0]

  \pgfmathsetmacro{\zv}{1.0}
  \pgfmathsetmacro{\zother}{-1.0/3.0}
  \pgfmathsetmacro{\rxy}{2*sqrt(2)/3}   

  \pgfmathsetmacro{\angA}{0}
  \pgfmathsetmacro{\angB}{120}
  \pgfmathsetmacro{\angC}{240}

  \coordinate (O)  at (0,0,0);
  \coordinate (v1) at (0,0,\zv);

  \coordinate (v2) at ({\rxy*cos(\angA)},{\rxy*sin(\angA)},{\zother});
  \coordinate (v3) at ({\rxy*cos(\angB)},{\rxy*sin(\angB)},{\zother});
  \coordinate (v4) at ({\rxy*cos(\angC)},{\rxy*sin(\angC)},{\zother});

  \coordinate (u2) at ({\rxy*cos(\angA)},{\rxy*sin(\angA)},0);
  \coordinate (u3) at ({\rxy*cos(\angB)},{\rxy*sin(\angB)},0);
  \coordinate (u4) at ({\rxy*cos(\angC)},{\rxy*sin(\angC)},0);

  \filldraw[fill=gray!15, draw=gray!40]
    (-1.2,-1.0,0) -- (1.2,-1.0,0) -- (1.2,1.0,0) -- (-1.2,1.0,0) -- cycle;
  \node[gray!60] at (0.0,-1.05,0) {$v_1^\perp\ (z=0)$};

  \draw[gray!50,->] (O) -- (1.25,0,0) node[below] {$x$};
  \draw[gray!50,->] (O) -- (0,1.05,0) node[left] {$y$};
  \draw[gray!50,->] (O) -- (0,0,1.25) node[right] {$z$};

  \draw[gray!60] (v2) -- (v3) -- (v4) -- cycle;
  \draw[gray!60] (v1) -- (v2);
  \draw[gray!60] (v1) -- (v3);
  \draw[gray!60] (v1) -- (v4);

  \draw[->,thick] (O) -- (v1) node[pos=0.6,left] {$v_1$};
  \draw[->,thick] (O) -- (v2) node[pos=0.55,below right] {$v_2$};
  \draw[->,thick] (O) -- (v3) node[pos=0.55,below left] {$v_3$};
  \draw[->,thick] (O) -- (v4) node[pos=0.55,left] {$v_4$};

  \draw[densely dashed] (v2) -- (u2);
  \draw[densely dashed] (v3) -- (u3);
  \draw[densely dashed] (v4) -- (u4);

  \draw[->,thick,blue!70] (O) -- (u2) node[below right=-2pt] {$u_2$};
  \draw[->,thick,blue!70] (O) -- (u3) node[below left=-2pt] {$u_3$};
  \draw[->,thick,blue!70] (O) -- (u4) node[below] {$u_4$};

  \draw[blue!60,thick] (u2) -- (u3) -- (u4) -- cycle;

  \node[gray!60] at (0.05,0.85,0.95) {$\angle(v_i,v_j)=\arccos(-\tfrac{1}{3})$};

\end{tikzpicture}
\caption{C=4 simplex ETF. One vertex $v_1$ along $+z$; the other three lie at $z=-1/3$ with equal $120^\circ$ separation in $xy$. Orthogonal projection onto $v_1^\perp$ ($z=0$) yields an equilateral triangle formed by $u_2,u_3,u_4$.}
\label{fig:proj_etf}
\vspace{-0.3cm}
\end{figure}
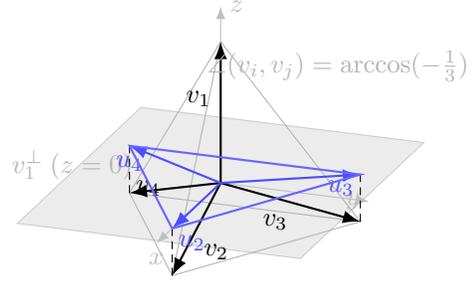

A proof of Proposition~\ref{prop:etf-projection} is given in Appendix~\ref{app:proof-etf-projection}.
This invariance implies that class forgetting via orthogonal projection maintains perfect angular separation among retained classes, forming the geometric basis of our POUR method.

\section{Method}
\begin{figure*}[t]
    \centering      
    \includegraphics[width=0.85\linewidth]{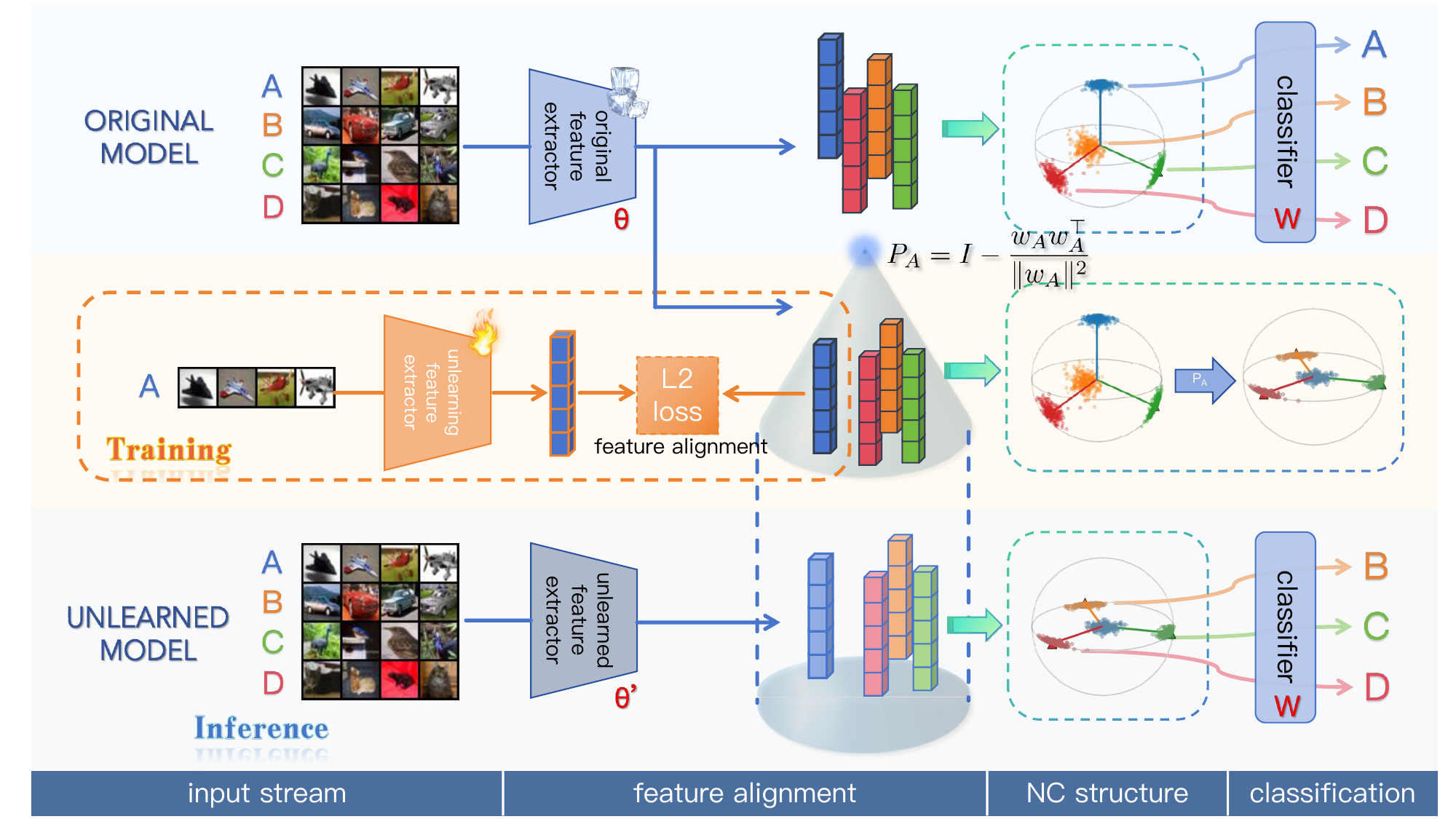} 
    \caption{\textbf{Overview of the POUR framework.}  
During training, the unlearning module applies an orthogonal projection operator $P_A$ on the feature space of the original model to remove the contribution of the forgotten class $A$.  
The unlearned feature extractor $\theta'$ is optimized via an $L_2$ loss to align its projected features with those of the original extractor $\theta$ using the unlearning data.  
This alignment preserves the Neural Collapse geometry among retained classes ($B$, $C$, $D$) while collapsing features of the forgotten class to the origin, leading to uniform predictions.  
At inference, the unlearned model is Bayes-optimal on retained classes as proved in Theorem \ref{thm:opt}.  
} 
    \label{fig:main_plot} 
    \vspace{-0.5cm}
\end{figure*}

Inspired by these theoretical insights, we propose our method POUR, which comes in two variants: a one-shot projection operation on model weights enabled by Proposition \ref{prop:etf-projection}, which we call \textbf{POUR-P}; and a distillation version using the forget set, which we call \textbf{POUR-D}, as summarized in Fig. \ref{fig:main_plot}.

\subsection{Projection Operator (POUR-P)}

We consider the class forgetting setting, where we want to forget a class $u$, so that $\mathcal{D}_f=\mathcal{D}_u:=\{x_i:y_i=u\}$.
Let $W \in \mathbb{R}^{d \times C}$ denote the classifier weight matrix, where each column $w_c$ corresponds to class $c$, and let $z \in \mathbb{R}^d$ denote the penultimate-layer feature. 
In the NC regime, both $\{w_c\}_{c=1}^C$ and the class means of $\{z\}$ form a simplex ETF.
To forget a class $u$, we define the orthogonal projection operator
\begin{equation}\label{eq:projector}
    P = I - \frac{w_u w_u^\top}{\|w_u\|^2},
\end{equation}
which removes the contribution of the forgotten class direction $w_u$. 
The unlearned features can then be obtained by
\begin{equation}
z' = P^\top z=Pz.
\end{equation}
Directly applying this projection after the feature extractor is what we call \textbf{POUR-P}.
By Proposition~\ref{prop:etf-projection}, this operation maps features into a $(C-1)$-class simplex ETF subspace, thereby preserving optimal geometry among the retained classes.

\noindent \textbf{Practical estimation of $P$.}
When classifier weights are not directly available, or when only the feature encoder is available (e.g., for vision-language models), we can estimate $w_u$ as the empirical class mean of penultimate features,
\begin{equation}
\tilde{w}_u = \frac{1}{|\mathcal{D}_u|} \sum_{x \in \mathcal{D}_u} \theta_o(x),
\end{equation}
where $\theta_o$ denotes the original feature extractor. 
The projection operator $P$ can then be constructed using $\tilde{w}_u$.

\definecolor{skyblue}{RGB}{173,216,230}

\begin{table*}[t]
\centering
\caption{
\textbf{Comparison of unlearning methods for ResNet18 on CIFAR-10.}  
We report both {classification-level} metrics and representation-level metrics.  
Methods requiring access to the retain set are included as reference values and not included in ranking.  
Values represent mean~$\pm$~std across three runs.  
Best values are in \textbf{bold}, and second-best values are \underline{underlined}.  
\colorbox{skyblue!90}{Darker blue} indicates better performance.
}
\vspace{-2mm}
\label{tab:cifar10}
\resizebox{0.97\textwidth}{!}{
\begin{tabular}{l|ccccc|ccccccc}
\toprule
\multirow{2}{*}{Method} & \multicolumn{5}{c|}{Classification-Level Metrics} & \multicolumn{7}{c}{Representation-Level Metrics} \\
\cmidrule(lr){2-6} \cmidrule(lr){7-13}
 & Acc$_r \uparrow$ & Acc$_f \downarrow$ & Acc$_{tr} \uparrow$ & Acc$_{tf} \downarrow$ & AUS $\uparrow$ 
 & rMIA $\downarrow$ & CKA$_f^{(o)} \downarrow$ & CKA$_r^{(o)} \uparrow$ & RUS$^{(o)}\uparrow$ & CKA$_f^{(r)} \uparrow$ & CKA$_r^{(r)} \uparrow$ & RUS$^{(r)}\uparrow$ \\
\midrule
Original Model & 94.47{\tiny ±0.12} & 95.03{\tiny ±0.35} & 99.99{\tiny ±0.00} & 99.99{\tiny ±0.01} & \cellcolor{skyblue!15}{0.51{\tiny ±0.00}} & 56.70{\tiny ±0.00} & 1.00{\tiny ±0.00} & 1.00{\tiny ±0.00} & 0.00{\tiny ±0.00} & 0.26{\tiny ±0.01} & 0.98{\tiny ±0.00} & \cellcolor{skyblue!50}{0.42{\tiny ±0.01}} \\
Retrained Model & 94.68{\tiny ±0.38} & 0.00{\tiny ±0.00} & 99.98{\tiny ±0.01} & 0.00{\tiny ±0.00} & \cellcolor{skyblue!90}{1.00{\tiny ±0.00}} & -- & 0.26{\tiny ±0.03} & 0.97{\tiny ±0.01} & \cellcolor{skyblue!90}{0.84{\tiny ±0.01}} & 1.00{\tiny ±0.00} & 1.00{\tiny ±0.00} & \cellcolor{skyblue!100}{1.00{\tiny ±0.00}} \\
\midrule
\rowcolor{gray!20}
\multicolumn{13}{c}{\textit{Methods requiring the retain set}} \\
\midrule
\gtext{Finetune} & \gtext{93.96{\tiny ±0.19}} & \gtext{0.00{\tiny ±0.00}} & \gtext{100.00{\tiny ±0.00}} & \gtext{0.00{\tiny ±0.00}} & \cellcolor{skyblue!85!white!80}{\gtext{0.99{\tiny ±0.00}}} & \gtext{54.60{\tiny ±1.37}} & \gtext{0.32{\tiny ±0.01}} & \gtext{0.97{\tiny ±0.01}} & \cellcolor{skyblue!70!white!80}{\gtext{0.80{\tiny ±0.00}}} & \gtext{0.63{\tiny ±0.02}} & \gtext{0.97{\tiny ±0.00}} & \cellcolor{skyblue!90!white!80}{\gtext{0.76{\tiny ±0.02}}} \\
\gtext{FCS} & \gtext{94.89{\tiny ±0.01}} & \gtext{0.67{\tiny ±0.55}} & \gtext{100.00{\tiny ±0.00}} & \gtext{0.79{\tiny ±0.67}} & \cellcolor{skyblue!88!white!80}{\gtext{1.00{\tiny ±0.00}}} & \gtext{56.00{\tiny ±0.61}} & \gtext{0.51{\tiny ±0.09}} & \gtext{1.00{\tiny ±0.02}} & \cellcolor{skyblue!25!white!80}{\gtext{0.66{\tiny ±0.08}}} & \gtext{0.45{\tiny ±0.01}} & \gtext{0.98{\tiny ±0.00}} & \cellcolor{skyblue!65!white!80}{\gtext{0.62{\tiny ±0.01}}} \\
\midrule
\rowcolor{gray!20}
\multicolumn{13}{c}{\textit{Methods on forget set only}} \\
\midrule
Random Label & 87.42{\tiny ±0.54} & 23.20{\tiny ±0.44} & 93.13{\tiny ±0.44} & 25.13{\tiny ±0.34} & \cellcolor{skyblue!25}{0.75{\tiny ±0.00}} & 54.07{\tiny ±1.31} & 0.29{\tiny ±0.05} & \underline{0.86{\tiny ±0.02}} & \cellcolor{skyblue!30}{0.78{\tiny ±0.02}} & 0.24{\tiny ±0.01} & \underline{0.84{\tiny ±0.00}} & \cellcolor{skyblue!32}{0.37{\tiny ±0.01}} \\
Gradient Ascent & 86.71{\tiny ±4.04} & 15.37{\tiny ±4.13} & 93.51{\tiny ±4.00} & 16.29{\tiny ±3.54} & \cellcolor{skyblue!30}{0.80{\tiny ±0.01}} & \textbf{50.40{\tiny ±0.82}} & \textbf{0.21{\tiny ±0.06}} & 0.80{\tiny ±0.07} & \cellcolor{skyblue!32}{0.79{\tiny ±0.02}} & 0.18{\tiny ±0.01} & 0.77{\tiny ±0.05} & \cellcolor{skyblue!20}{0.29{\tiny ±0.02}} \\

Boundary Shrink & 85.30{\tiny ±1.66} & 12.33{\tiny ±2.14} & 90.81{\tiny ±1.38} & 13.96{\tiny ±2.12} & \cellcolor{skyblue!32}{0.81{\tiny ±0.01}} & 53.07{\tiny ±1.10} & {0.25{\tiny ±0.04}} & 0.85{\tiny ±0.02} & \cellcolor{skyblue!65}{\underline{0.80{\tiny ±0.01}}} & \underline{0.28{\tiny ±0.01}} & \underline{0.84{\tiny ±0.00}} & \cellcolor{skyblue!50}{\underline{0.42{\tiny ±0.01}}} \\
Boundary Expand & 85.74{\tiny ±0.55} & 14.63{\tiny ±0.06} & 91.21{\tiny ±0.44} & 16.66{\tiny ±0.37} & \cellcolor{skyblue!32}{0.80{\tiny ±0.00}} & 53.00{\tiny ±0.78} & {0.25{\tiny ±0.04}} & 0.85{\tiny ±0.02} & \cellcolor{skyblue!65}{\underline{0.80{\tiny ±0.01}}} & \underline{0.28{\tiny ±0.01}} & 0.83{\tiny ±0.00} & \cellcolor{skyblue!50}{\underline{0.42{\tiny ±0.01}}} \\
DELETE & 88.73{\tiny ±2.32} & 2.43{\tiny ±0.12} & 95.43{\tiny ±1.78} & 2.93{\tiny ±0.46} & \cellcolor{skyblue!50}{0.92{\tiny ±0.02}} & 53.43{\tiny ±0.40} & 0.37{\tiny ±0.09} & 0.82{\tiny ±0.03} & \cellcolor{skyblue!23}{0.71{\tiny ±0.05}} & 0.26{\tiny ±0.01} & 0.78{\tiny ±0.02} & \cellcolor{skyblue!35}{0.39{\tiny ±0.01}} \\
\midrule
POUR-P\textsuperscript{\dag} (ours) & \textbf{94.97{\tiny ±0.16}} & \textbf{0.00{\tiny ±0.00}} & \textbf{99.99{\tiny ±0.00}} & \textbf{0.00{\tiny ±0.00}} & \cellcolor{skyblue!90}{\textbf{1.01{\tiny ±0.00}}} & 56.67{\tiny ±1.08} & -- & -- & -- & -- & -- & -- \\

POUR-D (ours) & \underline{92.86{\tiny ±1.02}} & \underline{0.37{\tiny ±0.64}} & \underline{99.74{\tiny ±0.29}} & \underline{0.43{\tiny ±0.44}} & \cellcolor{skyblue!70}{\underline{0.97{\tiny ±0.00}}} & \underline{51.80{\tiny ±2.42}} & \underline{0.23{\tiny ±0.06}} & \textbf{0.95{\tiny ±0.02}} & \cellcolor{skyblue!90}{\textbf{0.85{\tiny ±0.03}}} & \textbf{0.31{\tiny ±0.01}} & \textbf{0.94{\tiny ±0.00}} & \cellcolor{skyblue!80}{\textbf{0.47{\tiny ±0.01}}} \\
\bottomrule
\end{tabular}
}
\vspace{0mm}
\parbox{0.9\textwidth}{\centering\footnotesize
\textsuperscript{\dag}~POUR-P does not modify the encoder representations; therefore, representation-level metrics would be unchanged and are omitted.
}
\vspace{-0.4cm}
\end{table*}

\subsection{Projection-Guided Distillation (POUR-D)}

While POUR-P provides an immediate, closed-form forgetting operation, it only modifies features post hoc. 
To induce forgetting deeper into the feature extractor and improve robustness, we introduce a teacher-student distillation~\cite{hinton2015distilling} scheme, \textbf{POUR-D}.

\noindent \textbf{Teacher construction.}
We use the projected model POUR-P as the teacher. 
Given a trained model $(\theta, W)$ and a forget class $u$, we apply the projection operator $P$ from Equation~\ref{eq:projector} to obtain a projected teacher model $(P \theta, W)$ which encodes the post-forgetting ETF geometry in the representation space.  

\noindent \textbf{Student training.}
The student model finetunes the feature extractor parameters on the forget set to align with the teacher model. In particular, for $\theta$ the feature extractor of the original model and $\theta_s$ the feature extractor of the student model,  we minimize the L2 loss defined as:
\begin{equation*}
\mathcal{L}_{\text{POUR-D}}(x) 
= \|\theta_s(x) - P\theta(x)\|_2^2,\quad x\in \mathcal{D}_f.
\end{equation*}

Under NC, the class means form a simplex ETF, and the classifier head aligns with them. 
Projection preserves this ETF structure for retained classes. 
This loss penalizes deviations from the projected ETF features, ensuring that the student model remains aligned in both direction and scale with the teacher, while requiring minimal updates to the feature extractor parameters.
Convergence can be guaranteed by the following proposition. 

\begin{proposition}[L2 convergence implies CKA convergence]
\label{prop:l2-to-cka}
Let $Z,T\in\mathbb{R}^{n\times p}$ be row-centered, and assume $TT^\top\neq 0$.
If $\|Z-T\|_F \to 0$, then
$
\mathrm{CKA}(Z,T) \to 1.
$
\end{proposition}

\subsection{Optimality of Projection for Weak Unlearning}
\label{sec:opt-proj}

We now show that the proposed projection operator is optimal under the definition of
representation-level weak unlearning (Def.~\ref{def:rep-unlearn}). 
Projecting onto the orthogonal complement of the forgotten class removes its contribution
while preserving the Bayes-optimal ETF geometry of the retained classes.

\begin{theorem}[Optimality of POUR-P]
\label{thm:opt}
Assume (A1)--(A5) as in Appendix \ref{app:nc_assumptions} in the model training pipeline, and class priors are balanced and isotropic Gaussians, as described in Section \ref{subsec:class_assump}. Let $\theta(x)$ denote the penultimate layer features, $v_i$ the class means, then by NC, we have $\theta(x)\mid(y=i)\sim\mathcal N(v_i,\sigma^2 I_d)$, 
where $\{v_i\}_{i=1}^{C}$ form a simplex ETF.  

Now fix a class $u\in\mathcal Y$ and define the orthogonal projection 
$P=I-v_u v_u^\top$, projected features $\theta'(x)=P\theta(x)$, 
and $\tilde v_i=P v_i/\|P v_i\|$ for $i\neq u$. 
Then:

\begin{enumerate}[label=(\alph*)]
\item \textbf{Retained optimality and ETF equivalence.} The projected means $\{\tilde v_i\}_{i\neq u}$ form a simplex ETF (Prop.~\ref{prop:etf-projection}).
The retained-class ETF of the projected model and that of $M_r$ differ by at most an orthogonal transform, i.e., for any discrepancy $\mathcal K$ invariant under orthogonal transforms and rescaling, $\mathcal K(P_{\neg u},Q_{\neg u})=0.$

Moreover, under the Gaussian model, when class means dominate intra-class variability, the projected model is Bayes-optimal (Prop.~\ref{prop:etf-opt}).

\item \textbf{Complete forgetting.}  
Since $P v_u=0$, features of the forgotten class satisfy 
$\theta'(x)\!\mid\!(y=u)\!\sim\!\mathcal N(0,\sigma^2 P)$. 
In the low-variance (NC) limit $\sigma^2\!\to\!0$, 
$\theta'(x)\!\to\!0$ and all retained logits vanish, 
so $q_{\neg u}(\cdot|x)\!\to\!U_{\neg u}$:  
the forgotten class is represented by a uniform predictive distribution among the retained classes, i.e., $\alpha=0$.
\end{enumerate}
\end{theorem}

Complete statement and proof are included in Appendix \ref{app:main_theorem}.
Consequently, the POUR-P projection yields a representation that is (i) Bayes-optimal on $\mathcal Y_{\neg u}$ and (ii) representation-equivalent to the retrained model up to orthogonal gauge freedom, so that the representation-level discrepancy $\mathcal K(P_z^{(f)},P_z^{(f)})$ attains its minimum under Def.~\ref{def:rep-unlearn}.

\section{Experiments and Results}

We evaluate both POUR-P and POUR-D on CIFAR-10/100 with ResNet-18 and PathMNIST with pretrained ViT-S/16. 

\subsection{Experimental Setup}
\label{subsec:exp-setup}

\noindent \textbf{Protocol constraints.}
We follow the standard unlearning setting in recent literature \citep{zhou2025decoupled}, that assumes
(i) no access to the retained set $\mathcal{D}_r$ during unlearning, and 
(ii) no intervention in the original training procedure. 
All methods are applied directly to the already trained model. 


\noindent \textbf{Datasets and Models.}
For {CIFAR-10} and {CIFAR-100}, we implement modified {ResNet-18} backbones in which the initial $7{\times}7$ convolution (stride 2) is replaced by a $3{\times}3$ convolution (stride 1) and the subsequent max-pooling layer is removed to better suit $32{\times}32$ inputs. 
For {PathMNIST}~\cite{yang2023medmnist}, we tested in a pretraining setting, where we loaded a {ViT-S/16} pretrained on ImageNet and trained a classifier head. We evaluate performance on both internal and external test sets for the same task, which exhibits a domain shift. 

\noindent \textbf{Baselines.}
We benchmark POUR-P and POUR-D against a diverse set of existing unlearning strategies, including Finetune, FCS~\citep{cadet2024deep}, Random Label~\citep{graves2021amnesiac}, Gradient Ascent~\citep{graves2021amnesiac}, Boundary Shrink, Boundary Expand~\citep{chen2023boundary}, and DELETE~\citep{zhou2025decoupled}. 
Original Model and Retrain Model serve as the lower and upper bounds, respectively.

\noindent \textbf{Metrics.}
We report the following metrics:
\begin{itemize}\setlength{\itemsep}{2pt}
\item \textbf{Acc$_f$}, \textbf{Acc$_{tf}$} (\% $\downarrow$) and \textbf{Acc$_r$}, \textbf{Acc$_{tr}$} (\% $\uparrow$):
validation and training accuracy on the forget and retain sets, respectively.
\item \textbf{Adaptive Unlearning
Score (AUS) } ($\uparrow$) \citep{cotogni2023duck, li2025machine} :
Jointly capture retention and forgetting accuracy, defined as
$
\text{AUS} = \frac{1 - drop_r }{1 + acc_f},
$
where $drop_r$, and $acc_f$ denote the drop in retain accuracy and forgotten-class accuracy. 

\item \textbf{rMIA} (\% $\downarrow$):
representation-level membership-inference attack success rate on $\mathcal{D}_f$. We perform a five-fold attack using a linear regressor on the representation between the train and test sets. 
\item \textbf{CKA$_f^{(o)}$}, \textbf{CKA$_r^{(o)}$, RUS$^{(o)}$} ($\downarrow$, $\uparrow$, $\uparrow$) and \textbf{CKA$_f^{(r)}$}, \textbf{CKA$_r^{(r)}$, RUS$^{(r)}$} ($\uparrow$, $\uparrow$, $\uparrow$): as defined in Sect. \ref{sec:cka}.

\end{itemize}

\subsection{Unlearning on CIFAR-10/100}
\label{subsec:cifar100}

{On CIFAR-10}, as shown in Table \ref{tab:cifar10}, our method achieves the best performance for both the classification-level metric (AUS) and the representation-level metrics (RUS), suggesting that POUR enables efficient forgetting directly in representation space. 
In contrast, other methods, such as Gradient Ascent, Boundary Shrink, and DELETE, show lower AUS or RUS scores, indicating that their forgetting is either incomplete or occurs primarily at the classifier layer without sufficiently modifying the underlying representation geometry, as visualized in Figure \ref{fig:tsne_cifar10}. 
This also provide a parallel comparison of the two variants of the RUS. In general, these two scores establish a similar trend. 

\definecolor{skyblue}{RGB}{173,216,230}

\begin{table}[t]

\centering
\caption{\textbf{Comparison of unlearning methods for ResNet18 on CIFAR-100.} Best values are in \textbf{bold}, and second-best values are \underline{underlined}.  
\colorbox{skyblue!90}{Darker blue} indicates better performance.}
\vspace{-2mm}
\label{tab:cifar100}
\resizebox{0.49\textwidth}{!}{
\begin{tabular}{l|ccccccc}
\toprule
Method & Acc$_{r}$ $\uparrow$ & Acc$_{f}$ $\downarrow$ & AUS $\uparrow$ & CKA$_{f}^{(r)}$ $\uparrow$ & CKA$_{r}^{(r)}$ $\uparrow$ & RUS$^{(r)}$ $\uparrow$ & rMIA $\downarrow$ \\
\midrule
Original Model & 77.69 & 92.00 & \cellcolor{skyblue!20}0.52 & 0.60 & 0.78 & \cellcolor{skyblue!60}0.68 & 62.00 \\
Retrained Model  & 76.28 & 0.00 & \cellcolor{skyblue!90}1.00 & 1.00 & 1.00 & \cellcolor{skyblue!90}1.00 & -- \\
\midrule
\rowcolor{gray!20}
\multicolumn{8}{c}{\textit{Methods requiring the retain set}} \\
\midrule
\gtext{Finetune} & \gtext{76.32} & \gtext{0.00} & \cellcolor{skyblue!83}\gtext{0.99} & \gtext{0.57} & \gtext{0.76} & \cellcolor{skyblue!57}\gtext{0.67} & \gtext{54.00} \\
\gtext{FCS} & \gtext{76.81} & \gtext{2.00} & \cellcolor{skyblue!77}\gtext{0.97} & \gtext{0.61} & \gtext{0.78} & \cellcolor{skyblue!60}\gtext{0.68} & \gtext{55.00} \\
\midrule
\rowcolor{gray!20}
\multicolumn{8}{c}{\textit{Methods on forget set only}} \\
\midrule
Random Label & 61.98 & 11.00 & \cellcolor{skyblue!40}0.76 & 0.40 & 0.53 & \cellcolor{skyblue!23}0.46 & 49.00 \\
Gradient Ascent & 50.46 & 6.00 & \cellcolor{skyblue!25}0.69 & 0.44 & 0.44 & \cellcolor{skyblue!20}0.44 & 50.00 \\
Boundary Shrink & 68.87 & 4.00 & \cellcolor{skyblue!60}0.88 & \underline{0.55} & 0.72 & \cellcolor{skyblue!45}0.62 & 49.00 \\
Boundary Expand & 66.47 & 13.00 & \cellcolor{skyblue!47}0.79 & \underline{0.55} & \underline{0.74} & \cellcolor{skyblue!48}\underline{0.63} & \textbf{42.00} \\
DELETE & 64.67 & 8.00 & \cellcolor{skyblue!53}0.81 & 0.51 & 0.67 & \cellcolor{skyblue!37}0.58 & 60.00 \\
\midrule
POUR-P\textsuperscript{\dag} (ours) & \textbf{77.65} & \textbf{0.00} & \cellcolor{skyblue!90}\textbf{1.00} & \textbf{--} & \textbf{--} & \cellcolor{skyblue!60}\textbf{--} & 62.00 \\
POUR-D (ours) & \underline{73.44} & \underline{1.00} & \cellcolor{skyblue!80}\underline{0.95} & \textbf{0.57} & \textbf{0.76} & \cellcolor{skyblue!53}\textbf{0.65} & \underline{46.00} \\
\bottomrule
\end{tabular}
}
{\scriptsize
\textsuperscript{\dag}~POUR-P does not modify the encoder representations.
}
\vspace{-0.4cm}
\end{table}

\begin{figure}[t]
    \centering
    \begin{subfigure}[b]{0.9\linewidth}
        \centering
        \includegraphics[width=\linewidth, trim={3cm 2.8cm 12cm 1.4cm}, clip]{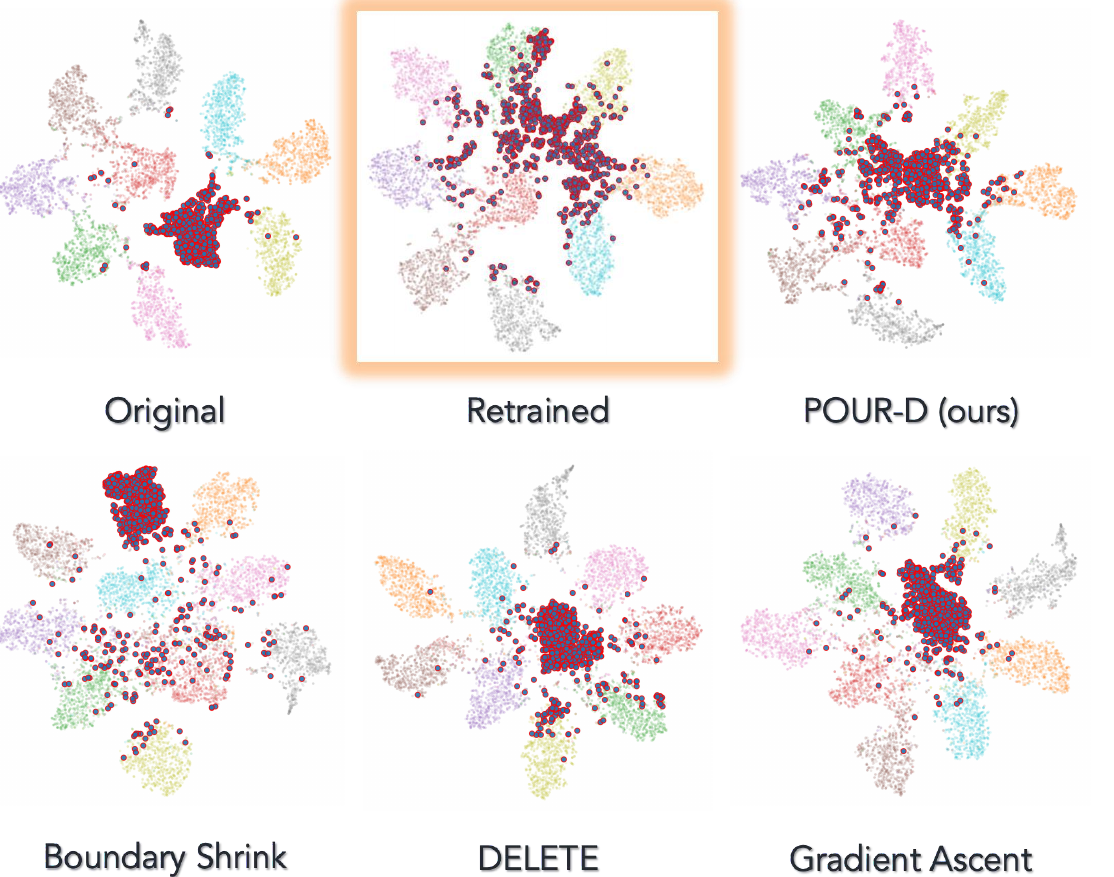}
        \caption{{t-SNE visualization on CIFAR10.} }
        \label{fig:tsne_cifar10}
    \end{subfigure}
    \hfill
    \begin{subfigure}[b]{0.9\linewidth}
        \centering
        \includegraphics[width=\linewidth, trim={1.6cm 1.2cm 9cm 0.3cm}, clip]{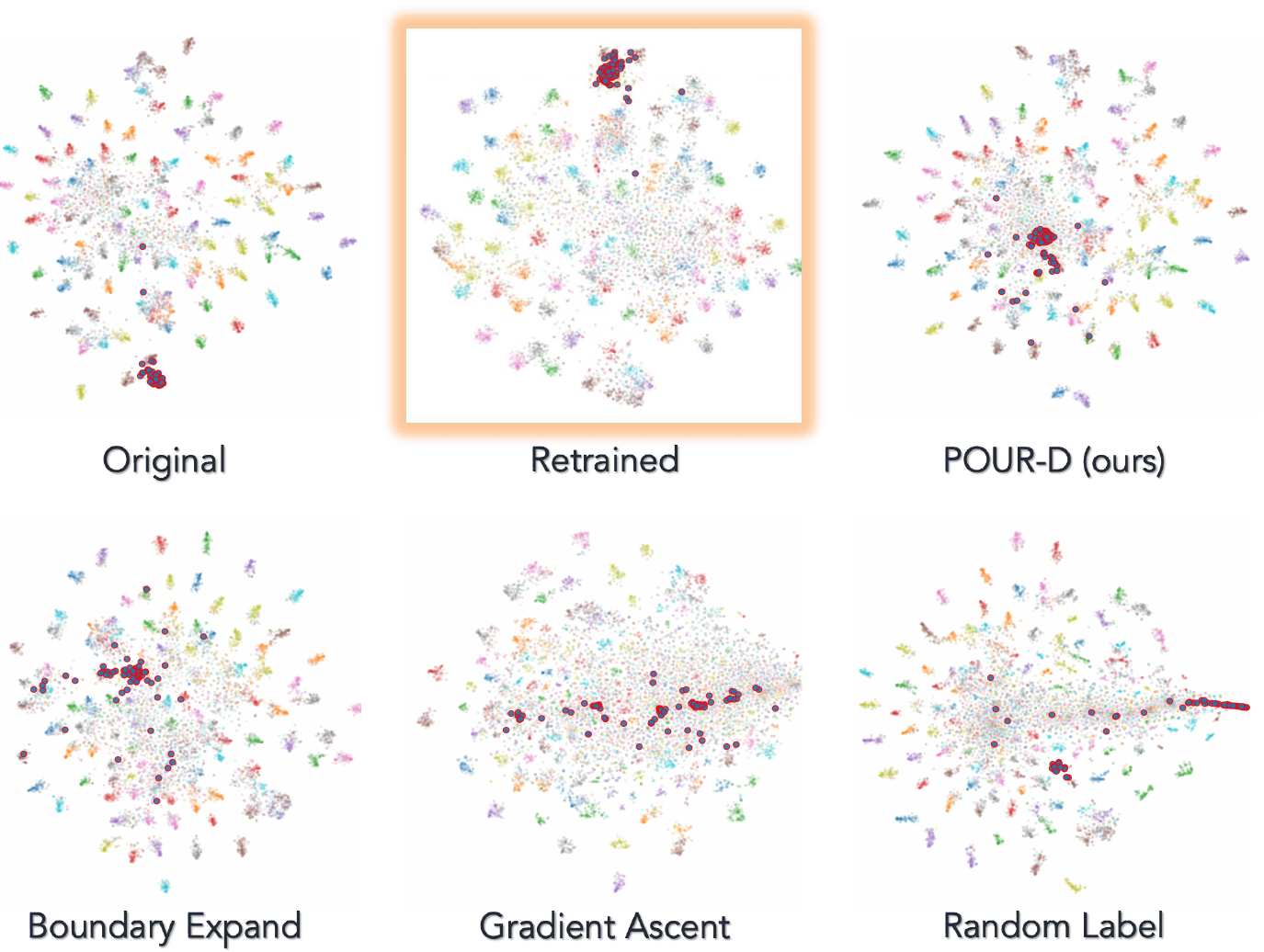}
        \caption{{t-SNE visualization on CIFAR100.} }
        \label{fig:tsne_cifar100}
    \end{subfigure}
    \vspace{-0.1cm}
    \caption{\textbf{t-SNE visualization of representation spaces after unlearning on CIFAR-10 and CIFAR-100.} Each color denotes a retained class, with dark red points represent the forgotten class. The Gold panel shows the representation of the retrained model, serving as one possible reference for successful unlearning. 
    Structure of representations after POUR unlearning mostly resemble that of the retrained gold model.}
    \label{fig:tsne}
    \vspace{-0.6cm}
\end{figure}

{On the more challenging CIFAR-100 dataset}, as shown in Table \ref{tab:cifar100}, our method again achieves state-of-the-art (SOTA) performance for both classification and representation-level metrics. 
We note that on CIFAR-100, classes are more entangled, as shown as a high CKA$_f^{(r)}$ and visualized in Figure \ref{fig:tsne_cifar100}. Therefore, supervision on the forget set is lower and therefore forgetting is harder, as discussed in Section \ref{subsec:theo_charac}. Methods such as gradient ascent and random labels largely disrupt the structure of the retained classes. Boundary Shrink and Boundary Expand, though among the stronger baselines, fail to reproduce the structure of the retrained model representations as effectively as POUR.

\definecolor{skyblue}{RGB}{173,216,230}

\begin{table*}[t]
\label{tab:main-results}
\centering
\caption{
\textbf{Comparison of unlearning methods on PathMNIST with ViT.}  
Performance is reported on both the \emph{internal} and \emph{external test} sets under domain shift.
Best values are in \textbf{bold}, and second-best values are \underline{underlined}.  
\colorbox{skyblue!90}{Darker blue} indicates better performance.
}
\vspace{-2mm}
\resizebox{0.7\textwidth}{!}{
\begin{tabular}{l|ccccc|ccccc}
\toprule
\multirow{2}{*}{Method} & \multicolumn{5}{c|}{ViT on PathMNIST Internal Test} & \multicolumn{5}{c}{ViT on PathMNIST External Test} \\
\cmidrule(lr){2-6} \cmidrule(lr){7-11}
& Acc$_r \uparrow$ & Acc$_f \downarrow$ & AUS $\uparrow$ & RUS $^{(o)}\uparrow$ & rMIA $\downarrow$
& Acc$_r \uparrow$ & Acc$_f \downarrow$ & AUS $\uparrow$ & RUS $^{(o)} \uparrow$ & rMIA $\downarrow$ \\
\midrule
Original Model & {87.49} & 96.83 & \cellcolor{skyblue!20}0.51 & -- & 50.43
& {87.13} & 97.53 & \cellcolor{skyblue!20}0.51 & -- & 84.20 \\

Retrained Model$^*$ & 88.70 & 0.00 & \cellcolor{skyblue!85}1.01 & -- & -- & 88.63 & 0.00 & \cellcolor{skyblue!85}1.02 & -- & -- \\
\midrule
\rowcolor{gray!20}
\multicolumn{11}{c}{\textit{Methods requiring the retain set}} \\
\midrule
\gtext{Finetune} & \gtext{97.78} & \gtext{0.00} & \cellcolor{skyblue!90}\gtext{1.10} & \cellcolor{skyblue!20}\gtext{0.05} & \gtext{46.67}
& \gtext{86.99} & \gtext{0.00} & \cellcolor{skyblue!85}\gtext{1.00} & \cellcolor{skyblue!25}\gtext{0.06} & \gtext{86.10} \\

\gtext{FCS} & \gtext{88.81} & \gtext{0.00} & \cellcolor{skyblue!85}\gtext{1.01} & \cellcolor{skyblue!30}\gtext{0.13} & \gtext{49.00}
& \gtext{83.86} & \gtext{0.00} & \cellcolor{skyblue!80}\gtext{0.97} & \cellcolor{skyblue!30}\gtext{0.10} & \gtext{96.60} \\

\midrule
\rowcolor{gray!20}
\multicolumn{11}{c}{\textit{Methods on forget set only$^*$}} \\
\midrule
Random Label & 78.71 & 23.73 & \cellcolor{skyblue!50}0.74 & \cellcolor{skyblue!40}0.26 & \textbf{48.23}
& 70.61 & 25.34 & \cellcolor{skyblue!45}0.67 & \cellcolor{skyblue!45}0.29 & \textbf{83.60} \\

Gradient Ascent & \underline{81.43} & 9.03 & \cellcolor{skyblue!60}0.86 & \cellcolor{skyblue!35}0.22 & \underline{48.71}
& 76.72 & 10.61 & \cellcolor{skyblue!58}0.81 & \cellcolor{skyblue!40}0.26 & \underline{83.90} \\



DELETE & 72.75 & \textbf{0.00} & \cellcolor{skyblue!58}0.85 & \underline{\cellcolor{skyblue!65}0.50} & {49.76}
& 73.04 & \textbf{0.00} & \cellcolor{skyblue!60}0.86 & \underline{\cellcolor{skyblue!58}0.42} & 88.40 \\

\midrule
POUR-P (ours) & \textbf{87.14} & \textbf{0.00} & \textbf{\cellcolor{skyblue!88}1.00} & \cellcolor{skyblue!15}-- & 50.43
& \textbf{87.44} & \textbf{0.00} & \textbf{\cellcolor{skyblue!85}1.00} & \cellcolor{skyblue!15}-- & {84.20} \\

POUR-D (ours) & 81.09 & \underline{7.88} & \underline{\cellcolor{skyblue!65}0.87} & \textbf{\cellcolor{skyblue!80}0.63} & 51.00
& \underline{80.90} & \underline{7.92} & \underline{\cellcolor{skyblue!63}0.87} & \textbf{\cellcolor{skyblue!75}0.61} & 85.20 \\
\bottomrule
\end{tabular}
}
\vspace{0mm}
\parbox{\textwidth}{\centering\footnotesize
*Note: Boundary Shrink and Boundary Expand did not work in this setting. Retraining only performed on classifier head.
}
\vspace{-0.7cm}
\end{table*}

\subsection{{Unlearning on PathMNIST}}

PathMNIST exhibits a substantial domain shift between its internal and external test sets due to differences in slide acquisition. The ViT backbone is pretrained on ImageNet with a finetuned classifier head on PathMNIST, which makes this setting challenging for existing unlearning methods, as the learning is shallow. 

Our method again achieves SOTA in this setting. As shown in Figure \ref{fig:gradcam}, the activation on the forget set vanishes after POUR-P unlearning, meaning unlearning signals successfully propagate from the finetuned classifier head into the pretrained backbone. Methods like Boundary Shrink and Boundary Expand fail in this setting.
Random Label and Gradient Ascent exhibit higher classification-level performance on the internal test set than on the external test set. 
We hypothesize that these methods are ``learning to mask'' the forget set rather than genuinely erasing the associated knowledge, resulting in poor generalization. 
In contrast, DELETE and POUR achieve consistent performance across both domains, suggesting that they perform true knowledge removal rather than overfitting to the forget set.
The effectiveness of the unlearning methods in this setting provides a scalable pathway for unlearning in pretrained and foundation models, where original data may be unavailable.

\definecolor{lightyellow}{rgb}{1.0, 0.97, 0.85}
\definecolor{lightblue}{rgb}{0.90, 0.95, 1.0}
\definecolor{darkgreen}{rgb}{0.0, 0.4, 0.0}
\definecolor{darkred}{rgb}{0.6, 0.0, 0.0}

\begin{table}[t]
\centering
\small
\renewcommand{\arraystretch}{1.1}
\vspace{-5pt}
\caption{POUR unlearning on \textbf{ImageNet} with \textbf{CLIP}.}
\label{tab:clip}
\vspace{-10pt}
\resizebox{0.98\linewidth}{!}{%
\begin{tabular}{lcccc}
\specialrule{1.5pt}{0pt}{0pt}
\multicolumn{1}{c}{\multirow{2}{*}{\textbf{Forgotten class}}} &
\multicolumn{2}{c}{\textbf{Before unlearning}} &
\multicolumn{2}{c}{\textbf{After unlearning}} \\
\cline{2-3}\cline{4-5}
& \cellcolor{lightyellow}\textbf{\ \ Acc$_r$\ \ }
& \cellcolor{lightblue}\textbf{Acc$_f$}
& \cellcolor{lightyellow}\textbf{Acc$_r$$\uparrow$ {($\Delta$)}}
& \cellcolor{lightblue}\textbf{Acc$_f$$\downarrow$ {($\Delta$)}} \\
\hline
Goldfish
& \cellcolor{lightyellow}67.36
& \cellcolor{lightblue}94.00
& \cellcolor{lightyellow}69.23\,\textcolor{darkgreen}{\scriptsize{(+1.87)}}
& \cellcolor{lightblue}14.00\,\textcolor{darkgreen}{\scriptsize{($-80.00$)}} \\
European fire salamander
& \cellcolor{lightyellow}67.36
& \cellcolor{lightblue}94.00
& \cellcolor{lightyellow}67.25\,\textcolor{darkred}{\scriptsize{($-0.11$)}}
& \cellcolor{lightblue}22.00\,\textcolor{darkgreen}{\scriptsize{($-72.00$)}} \\
Boa constrictor
& \cellcolor{lightyellow}67.53
& \cellcolor{lightblue}60.00
& \cellcolor{lightyellow}70.33\,\textcolor{darkgreen}{\scriptsize{(+2.80)}}
& \cellcolor{lightblue}6.00\,\textcolor{darkgreen}{\scriptsize{($-54.00$)}} \\
Centipede
& \cellcolor{lightyellow}67.62
& \cellcolor{lightblue}42.00
& \cellcolor{lightyellow}67.74\,\textcolor{darkgreen}{\scriptsize{(+0.12)}}
& \cellcolor{lightblue}10.00\,\textcolor{darkgreen}{\scriptsize{($-32.00$)}} \\
Bison
& \cellcolor{lightyellow}67.33
& \cellcolor{lightblue}100.00
& \cellcolor{lightyellow}67.76\,\textcolor{darkgreen}{\scriptsize{(+0.43)}}
& \cellcolor{lightblue}6.00\,\textcolor{darkgreen}{\scriptsize{($-94.00$)}} \\
\specialrule{1.5pt}{0pt}{0pt}
\end{tabular}%
}
\vspace{-8pt}
\end{table}

\subsection{NC as an Assumption}
\label{subsec:nc_pheno}

{Our analysis relies on the theory of NC, which typically emerges under sufficient overparameterization. We  also found in practice that standard training naturally reaches a regime where NC is sufficiently well established for POUR to be effective. Figure~\ref{fig:weight-angle-distributions} shows the classifier weight angle distributions across datasets, revealing how closely the trained models conform to NC geometry.
The empirical mean angles align almost perfectly with the ideal simplex ETF angle on all of the datasets.}

\begin{figure}[t]
    \centering
    \begin{subfigure}[t]{0.155\textwidth}
        \centering
        \includegraphics[height=1.6cm, width=\linewidth, trim=2 0 0 0, clip]{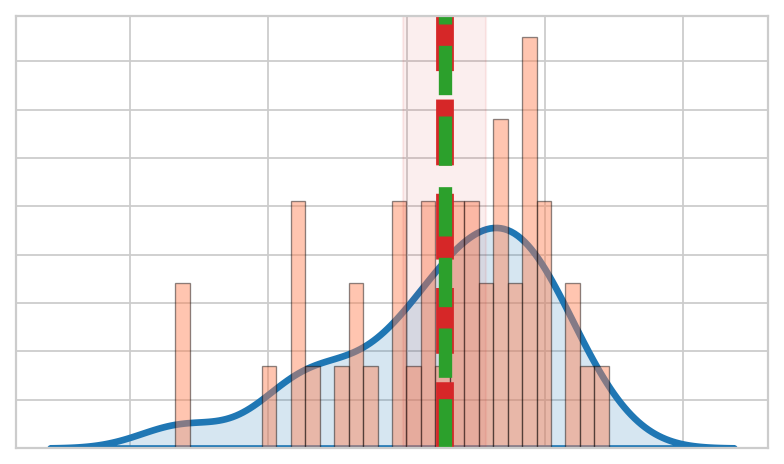}
        \caption{CIFAR-10}
        \label{fig:angle-cifar10}
    \end{subfigure}
    \hfill
    \begin{subfigure}[t]{0.155\textwidth}
        \centering
        \includegraphics[height=1.6cm, width=\linewidth, trim=12 0 0 0, clip]{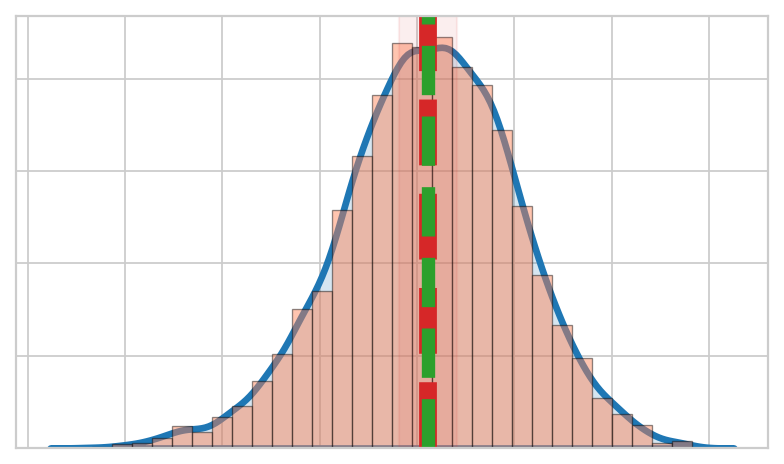}
        \caption{CIFAR-100}
        \label{fig:angle-cifar100}
    \end{subfigure}
    \hfill
    \begin{subfigure}[t]{0.155\textwidth}
        \centering
        \includegraphics[height=1.6cm, width=\linewidth, trim=2 0 0 0, clip]{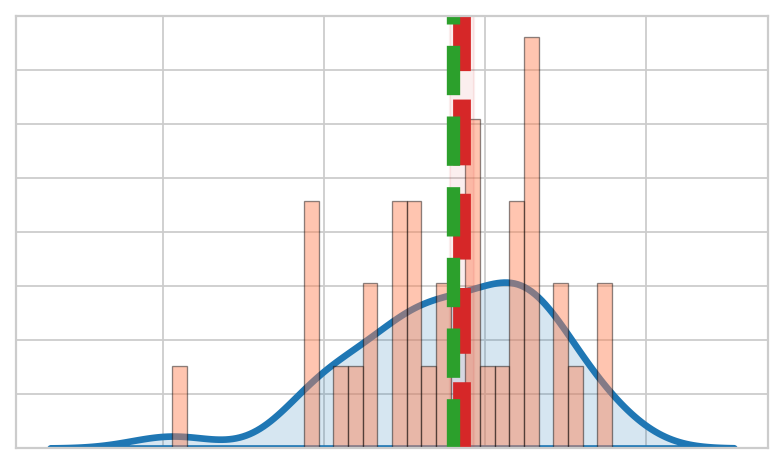}
        \caption{PathMNIST}
        \label{fig:angle-pathmnist}
    \end{subfigure}
    \vspace{-0.38cm}

    \caption{\textbf{Classifier weight angle distributions.} The green dashed line denotes the mean pairwise angle, while the red dashed line marks the ideal NC angle.
    The closeness between the two reflects how well the classifier aligns with NC geometry at convergence.}
    \vspace{-0.5cm}
    \label{fig:weight-angle-distributions}
\end{figure}

\subsection{Cross-Modal Unlearning}

Word embeddings in language or vision-language models can also exhibit near-isotropic structures~\cite{mikolov2013distributed}. We test the performance of POUR-P on CLIP-L/14~\cite{openai_clip_vit_large_patch14, radford2021learning} by removing the corresponding {text embeddings} from the 
{image encoder}. 
As shown in Figure \ref{fig:cifar10-zs-comparison} (Cifar-10) and Table \ref{tab:clip} (ImageNet), unlearning the text prompts reduces the accuracy of their associated classes while leaving the others largely unchanged.
This demonstrates that POUR remains effective in a cross-modal setting, illustrating the generality of its underlying geometric principles.

\subsection{Segmentation Unlearning}

Beyond classification, we evaluate POUR on a semantic segmentation task (VOC2012), where strong class imbalance arises from uneven pixel frequencies.
Tab.~\ref{tab:unlearning_results} and Fig.~\ref{fig:reb_img} show effective forgetting with preserved retained-class IoU.

\definecolor{lightyellow}{rgb}{1.0, 0.97, 0.85}
\definecolor{lightblue}{rgb}{0.90, 0.95, 1.0}
\definecolor{darkgreen}{rgb}{0.0, 0.4, 0.0}
\definecolor{navy}{rgb}{0.6, 0.0, 0.0}

\begin{table}[t!]
\centering
\small
\renewcommand{\arraystretch}{1.1}
\vspace{-5pt}
\caption{\textbf{Semantic segmentation unlearning} with POUR on VOC 2012 using \textbf{DeepLabV3+} with a \textbf{ResNet-101} backbone.}
\label{tab:unlearning_results}
\vspace{-10pt}
\resizebox{0.8\linewidth}{!}{%
\begin{tabular}{lcccc}
\specialrule{1.5pt}{0pt}{0pt}
\multicolumn{1}{c}{\multirow{2}{*}{\textbf{Forgotten class}}} &
\multicolumn{2}{c}{\textbf{Before unlearning}} &
\multicolumn{2}{c}{\textbf{After unlearning}} \\
\cline{2-3}\cline{4-5}
& \cellcolor{lightyellow}\textbf{\ \ IoU$_r$\ \ }
& \cellcolor{lightblue}\textbf{IoU$_f$}
& \cellcolor{lightyellow}\textbf{IoU$_r$$\uparrow$ {($\Delta$)}}
& \cellcolor{lightblue}\textbf{IoU$_f$$\downarrow$ {($\Delta$)}} \\
\hline
Dog
& \cellcolor{lightyellow}77.81
& \cellcolor{lightblue}90.98
& \cellcolor{lightyellow}71.53\,\textcolor{navy}{\scriptsize{($-6.28$)}}
& \cellcolor{lightblue}0.00\,\textcolor{darkgreen}{\scriptsize{($-90.98$)}} \\
Cat
& \cellcolor{lightyellow}77.53
& \cellcolor{lightblue}93.55
& \cellcolor{lightyellow}70.10\,\textcolor{navy}{\scriptsize{($-7.43$)}}
& \cellcolor{lightblue}0.00\,\textcolor{darkgreen}{\scriptsize{($-93.55$)}} \\
Bicycle
& \cellcolor{lightyellow}78.96
& \cellcolor{lightblue}41.26
& \cellcolor{lightyellow}77.58\,\textcolor{navy}{\scriptsize{($-1.38$)}}
& \cellcolor{lightblue}0.00\,\textcolor{darkgreen}{\scriptsize{($-41.26$)}} \\
Chair
& \cellcolor{lightyellow}78.98
& \cellcolor{lightblue}38.86
& \cellcolor{lightyellow}77.34\,\textcolor{navy}{\scriptsize{($-1.64$)}}
& \cellcolor{lightblue}0.81\,\textcolor{darkgreen}{\scriptsize{($-38.05$)}} \\
Bottle
& \cellcolor{lightyellow}78.26
& \cellcolor{lightblue}78.89
& \cellcolor{lightyellow}75.79\,\textcolor{navy}{\scriptsize{($-2.47$)}}
& \cellcolor{lightblue}0.00\,\textcolor{darkgreen}{\scriptsize{($-78.89$)}} \\
\specialrule{1.5pt}{0pt}{0pt}
\end{tabular}%
}
\vspace{-8pt}
\end{table}





\begin{figure}[t]
    \centering
    \vspace{-0mm}
    \includegraphics[width=\linewidth, height=1.8cm, trim=0 0 0 10, clip]{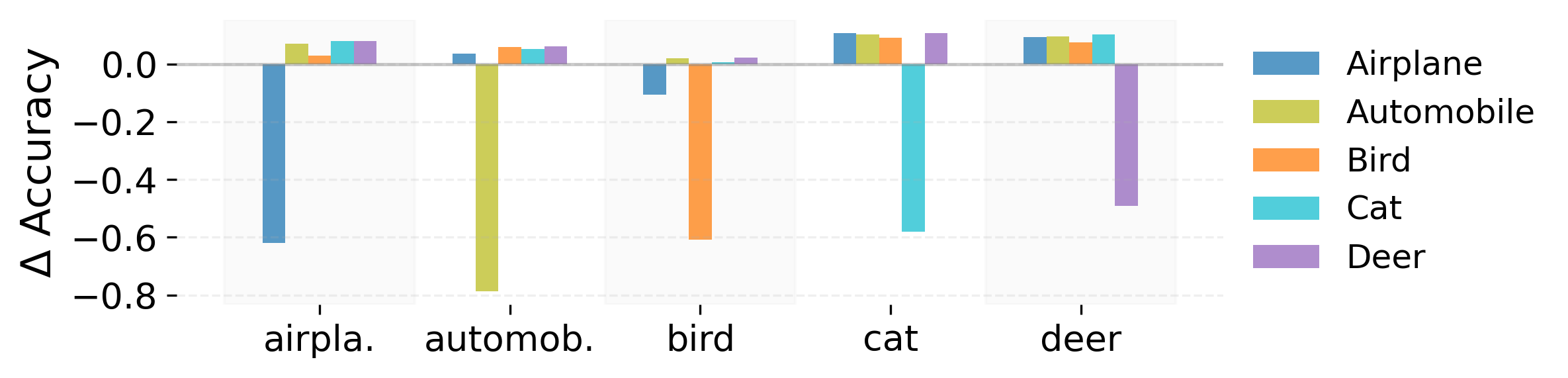}
    \vspace{-0.8cm}
    \caption{
    \textbf{Text-based zero-shot unlearning $\Delta$-accuracy on CLIP (CIFAR-10, first five classes) with POUR.}
    Each colour corresponds to unlearning a specific class with text prompts, and bars show the resulting change in classification accuracy for each class. 
    }


    \label{fig:cifar10-zs-comparison}
    \vspace{-0.3cm}
\end{figure}

\begin{figure}
    \centering
    \includegraphics[width=0.98\linewidth]{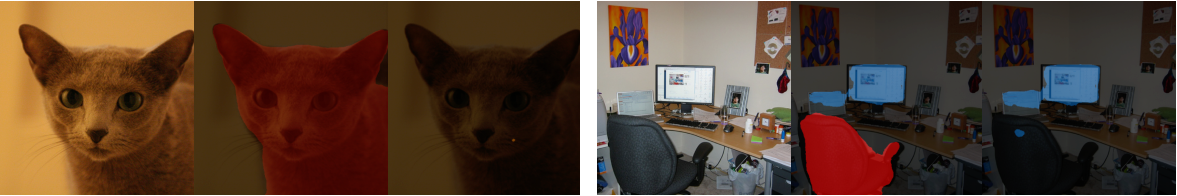}
    \vspace{-8pt}
\caption{Forgotten categories: \textbf{cat} (left) and \textbf{chair} (right).}
    \vspace{-18pt}
    \label{fig:reb_img}
\end{figure}

\section{Conclusion}

We introduced a representation-level formulation of machine unlearning and a new metric, RUS, to rigorously quantify forgetting beyond classifier logits. By connecting unlearning to NC geometry, we derived new theoretical insights which enabled POUR. Extensive experiments across CIFAR-10/100 and PathMNIST confirm that POUR achieves better forgetting than prior approaches at both the classification and representation levels, providing a principled geometric perspective on machine unlearning.

\newpage

\clearpage
{
    \small
    \bibliographystyle{ieeenat_fullname}
    \bibliography{main}
}

%

\clearpage
\setcounter{page}{1}
\setcounter{section}{0}

\maketitlesupplementary
\section*{Appendix Table of Contents}

\begin{enumerate}[label=\arabic*., start=0]
    \item Related Work
    \item Additional Baselines
    \item Additional Justifications
    \begin{enumerate}[label=\arabic{enumi}.\arabic*.]
        \item Proof on Decomposition of $\mathcal{K}$-Bound
        \item {Justification on CKA USage}
    \end{enumerate}

    \item Neural Collapse
    \begin{enumerate}[label=\arabic{enumi}.\arabic*.]
        \item Training Assumptions
        \item Neural Collapse Statements
    \end{enumerate}

    \item ETF Implies Bayes Optimality
    \begin{enumerate}[label=\arabic{enumi}.\arabic*.]
        \item Geometric Optimality of the Simplex ETF
        \item Bayes-Optimal Nearest Class Mean Rule
    \end{enumerate}

    \item Proof of Main Theorem
    \begin{enumerate}[label=\arabic{enumi}.\arabic*.]
        \item Closure of Projection
        \item Optimality of Projection
    \end{enumerate}

\end{enumerate}

\setcounter{section}{-1}
\setcounter{equation}{0}
\setcounter{figure}{0}
\vspace{-5mm}
\section{Related Work}
\paragraph{Machine Unlearning.}
The problem of removing specific training data from a model, often motivated by privacy regulations such as the ``right to be forgotten,'' was first formalized in the systems security community \citep{cao2015towards}. 
The seminal work of \citet{bourtoule2021machine} introduced the \emph{SISA} framework, partitioning training data across multiple shards so that forgetting can be achieved by retraining only the affected shards. 
Subsequent work developed more fine-grained methods that avoid full retraining. 
For linear models, \citet{guo2019certified} proposed certified removal via influence-based updates. 
\citet{sekhari2021remember} provided theoretical guarantees for approximate unlearning in general models. 
For deep networks, approaches include {amnesiac unlearning} \citep{graves2021amnesiac}, which inverts stored gradients, and Fisher information--based scrubbing \citep{golatkar2020eternal,golatkar2020forgetting}, which perturbs weights along sensitive directions. 
Other efficient methods use adversarial weight perturbations \citep{tarun2023fast}, incompetent teachers \citep{chundawat2023badteacher}, or zero-shot synthetic forget data \citep{chundawat2023zeroshot}. 
Most recently, anchored fine-tuning methods such as FAMR \citep{sanga2025famr} enforce uniform predictions on forget sets while constraining the model to remain close to its original parameters.
\citet{kodge2024deep} proposed a gradient-free method that explicitly computes class-specific subspaces via singular value decomposition and suppresses discriminatory directions associated with the forget class. 
{Boundary Shrink} and {Boundary Expand} \citep{chen2023boundary} perform local decision-boundary adjustments for forgetting, while maintaining model utility through margin control. 
{DELETE} \citep{zhou2025decoupled} formulates unlearning as a decoupled distillation problem, erasing class-specific information via probability decoupling.

\paragraph{Geometrically grounded forgetting.}
Several methods exploit the geometry of learned representations. 
\citet{kodge2024deep} proposed a gradient-free method that explicitly computes class-specific subspaces via singular value decomposition and suppresses discriminatory directions associated with the forget class. 
Yet, none of the previous approaches connects to the phenomenon of {Neural Collapse} \citep{papyan2020neuralcollapse}, wherein class features converge to a simplex equiangular tight frame.

\paragraph{Concept-level and multimodal unlearning.}
Beyond class forgetting, recent research has explored erasing visual concepts and multimodal associations. 
In generative models, concept erasure can be achieved by regularizing style features or Gram matrices \citep{zhang2024forgetmenot}. 
In multimodal settings, \citet{yang2025cliperase} proposed \emph{CLIPErase}, which disentangles forgetting, retention, and consistency modules to remove specific visual-textual alignments in CLIP. 
\citet{kravets2025zeroshot} introduced a zero-shot unlearning method for CLIP that generates synthetic forget samples via gradient ascent. 

\section{Additional Baselines} 
We additionally report results (forget airplane class) for method [22], SalUn~\cite{fan2023salun}, SCRUB~\cite{kurmanji2023towards} and SURE~\cite{sepahvand2025selective} in Tab.~\ref{tab:baseline}. POUR again achieves strongest performance across both classification- and representation-level metrics.

\begin{table}[t!]
\centering
\renewcommand{\arraystretch}{1.05}
\small
\caption{\textbf{More baselines} on CIFAR-10.}
\label{tab:baseline}
\vspace{-10pt}
\resizebox{0.98\linewidth}{!}{
\begin{tabular}{l c c c c c c}
\specialrule{1.5pt}{0pt}{0pt}
\textbf{Method} &
Acc$_r\uparrow$ &
Acc$_f\downarrow$ &
AUS$\uparrow$ &
CKA$_f\downarrow$ &
CKA$_r\uparrow$ &
RUS$^{(o)}\uparrow$ \\
\hline
Original          & 94.61 & 94.73 & 0.51 & 1.00 & 1.00 & 0.00 \\
ProjUn~[22]       & 86.49 & 20.40 & 0.76 & 0.34 & 0.89 & 0.76 \\
SalUn             & {92.72} & \underline{0.78}  & \textbf{0.97} & \underline{0.28} & 0.92 & \underline{0.81} \\
SCRUB             & \textbf{93.34} & 3.80  & \underline{0.95} & 0.31 & \textbf{0.98} & \underline{0.81} \\
SURE              & 92.50 & 2.60  & \underline{0.95} & 0.34 & \textbf{0.98} & 0.79 \\
\rowcolor{green!10}
{POUR}     & \underline{92.86} & \textbf{0.37}  & \textbf{0.97} & \textbf{0.23} & \underline{0.95} & \textbf{0.85 }\\
\specialrule{1.5pt}{0pt}{0pt}
\end{tabular}}
\vspace{-8pt}
\end{table}

\section{Additional Justifications}

\subsection{Proof on Decomposition of K-Bound}

\label{lem:decoupled-k}
Let $\mathcal{Z}$ denote the feature space and $\mathcal{P}(\mathcal{Z})$ the set of probability measures on it.
Fix a symmetric function class $\mathcal{F}\subseteq\{\varphi:\mathcal{Z}\!\to\!\mathbb{R}\}$ (i.e., $\varphi\!\in\!\mathcal{F}\Rightarrow -\varphi\!\in\!\mathcal{F}$).
For an Integral Probability Metric (IPM) defined as
\[
\mathcal{K}(P,Q)
=\sup_{\varphi\in\mathcal{F}}
\big|\mathbb{E}_{z\sim P}[\varphi(z)]-\mathbb{E}_{z\sim Q}[\varphi(z)]\big|,
\  P,Q\in\mathcal{P}(\mathcal{Z}),
\]
the following property holds.

\newpage
\begin{center}
\vspace*{-1cm}
\begin{minipage}{\textwidth}
\begin{proposition}[Decomposition of $\mathcal{K}$ Bound]
Fix a forgetting class $u$, and by the law of total probability, express the feature distributions as
\[
P_z=\alpha\,P_u+(1-\alpha)\,P_{\neg u},\qquad
Q_z=\beta\,Q_u+(1-\beta)\,Q_{\neg u},
\]
where $\alpha:=P(y{=}u)$ and $\beta:=Q(y{=}u)$ denote the class probabilities, and $P_{\neg u},Q_{\neg u}$ are the retained-class feature distributions.  
Let $\Delta_c = \mathcal{K}(P_u,P_{\neg u})$ denote the \emph{prior class separation} in the original model.  
Then the discrepancy between the unlearned and reference feature distributions is bounded as
\begin{align*}
&\big|\,\beta\,\mathcal{K}(P_u,Q_u)
 - (1-\beta)\,\mathcal{K}(P_{\neg u},Q_{\neg u})\,\big|
 - |\,\alpha-\beta\,|\,\Delta_c
\\[3pt]
&\hspace{2cm}\le\;
\mathcal{K}(P_z,Q_z)
\\[3pt]
&\hspace{2cm}\le\;
\underbrace{|\,\alpha-\beta\,|\,\Delta_c}_{\text{prior class separation}}
\;+\;
\underbrace{\beta\,\mathcal{K}(P_u,Q_u)}_{\text{forgotten-class alignment}}\;+\;
\underbrace{(1-\beta)\,\mathcal{K}(P_{\neg u},Q_{\neg u})}_{\text{retained-class alignment}}.
\end{align*}
\end{proposition}

\begin{proof}
For any $\varphi\in\mathcal{F}$, substituting in the decomposition, we have
\begin{align}
\mathbb{E}_{P_z}[\varphi]-\mathbb{E}_{Q_z}[\varphi]
&= \alpha\,\mathbb{E}_{P_u}[\varphi]+(1-\alpha)\,\mathbb{E}_{P_{\neg u}}[\varphi]
 - \beta\,\mathbb{E}_{Q_u}[\varphi]-(1-\beta)\,\mathbb{E}_{Q_{\neg u}}[\varphi]\\
&= (\alpha-\beta)\big(\mathbb{E}_{P_u}[\varphi]-\mathbb{E}_{P_{\neg u}}[\varphi]\big)
 + \beta\big(\mathbb{E}_{P_u}[\varphi]-\mathbb{E}_{Q_u}[\varphi]\big)
 + (1-\beta)\big(\mathbb{E}_{P_{\neg u}}[\varphi]-\mathbb{E}_{Q_{\neg u}}[\varphi]\big).
\label{eq:decomposition}
\end{align}

Taking absolute values and applying the triangle inequality yields
\[
\big|\mathbb{E}_{P_z}[\varphi]-\mathbb{E}_{Q_z}[\varphi]\big|
\le
|\alpha-\beta|\,\big|\mathbb{E}_{P_u}[\varphi]-\mathbb{E}_{P_{\neg u}}[\varphi]\big|
+\beta\,\big|\mathbb{E}_{P_u}[\varphi]-\mathbb{E}_{Q_u}[\varphi]\big|
+(1-\beta)\,\big|\mathbb{E}_{P_{\neg u}}[\varphi]-\mathbb{E}_{Q_{\neg u}}[\varphi]\big|.
\]
Now take the supremum over $\varphi\in\mathcal{F}$ on both sides.
Since $\mathcal{F}$ is symmetric, each term inside the absolute value corresponds exactly
to the IPM definition, i.e.,
\[
\sup_{\varphi\in\mathcal{F}}\big|\mathbb{E}_{P_u}[\varphi]-\mathbb{E}_{P_{\neg u}}[\varphi]\big|=\Delta_c, \qquad
\sup_{\varphi\in\mathcal{F}}\big|\mathbb{E}_{P_u}[\varphi]-\mathbb{E}_{Q_u}[\varphi]\big|=\mathcal{K}(P_u,Q_u),
\]
\[
\sup_{\varphi\in\mathcal{F}}\big|\mathbb{E}_{P_{\neg u}}[\varphi]-\mathbb{E}_{Q_{\neg u}}[\varphi]\big|
=\mathcal{K}(P_{\neg u},Q_{\neg u}).
\]
Hence,
\[
\mathcal{K}(P_z,Q_z)
=\sup_{\varphi\in\mathcal{F}}\big|\mathbb{E}_{P_z}[\varphi]-\mathbb{E}_{Q_z}[\varphi]\big|
\le
|\alpha-\beta|\,\Delta_c
+\beta\,\mathcal{K}(P_u,Q_u)
+(1-\beta)\,\mathcal{K}(P_{\neg u},Q_{\neg u}).
\]

For the lower bound, apply the reverse triangle inequality to Equation \ref{eq:decomposition}.
Let
\begin{equation}
x := (\alpha-\beta)\big(\mathbb{E}_{P_u}[\varphi]-\mathbb{E}_{P_{\neg u}}[\varphi]\big),
\quad
y := \beta\big(\mathbb{E}_{P_u}[\varphi]-\mathbb{E}_{Q_u}[\varphi]\big),
\quad
z := (1-\beta)\big(\mathbb{E}_{P_{\neg u}}[\varphi]-\mathbb{E}_{Q_{\neg u}}[\varphi]\big).
\end{equation}
Then
\begin{equation}
\big|\mathbb{E}_{P_z}[\varphi]-\mathbb{E}_{Q_z}[\varphi]\big|
\;\ge\;
\big|y+z\big|
\;-\;
|x|.
\label{eq:111}
\end{equation}
and by symmetry of $\mathcal{F}$ and the definition of $\Delta_c$,
\begin{equation}
|x|
\le
|\alpha-\beta|\,
\sup_{\varphi\in\mathcal{F}}
\big|\mathbb{E}_{P_u}[\varphi]-\mathbb{E}_{P_{\neg u}}[\varphi]\big|
=
|\alpha-\beta|\Delta_c.
\label{eq:222}
\end{equation}
Apply reverse triangle inequality again:
\begin{equation}
|y+z|
\;\ge\;
\Big|\,
\beta(\mathbb{E}_{P_u}[\varphi]-\mathbb{E}_{Q_u}[\varphi])
-(1-\beta)(\mathbb{E}_{P_{\neg u}}[\varphi]-\mathbb{E}_{Q_{\neg u}}[\varphi])
\,\Big|.
\end{equation}
Taking supremum over $\varphi\in\mathcal{F}$ and using symmetry:
\begin{equation}
\sup_{\varphi\in\mathcal{F}} |y+z|
\;\ge\;
\big|\,
\beta\,\mathcal{K}(P_u,Q_u)
-
(1-\beta)\,\mathcal{K}(P_{\neg u},Q_{\neg u})
\big|.
\label{eq:333}
\end{equation}

{Combining \eqref{eq:111}, \eqref{eq:222} and \eqref{eq:333}, we have}
\[
\mathcal{K}(P_z,Q_z)
=\sup_{\varphi\in\mathcal{F}}
\big|\mathbb{E}_{P_z}[\varphi]-\mathbb{E}_{Q_z}[\varphi]\big|
\ge
\big|\beta\,\mathcal{K}(P_u,Q_u)
-
(1-\beta)\,\mathcal{K}(P_{\neg u},Q_{\neg u})\big|
\;-\;
|\alpha-\beta|\,\Delta_c.
\]

This completes the proof.

\end{proof}

\end{minipage}
\end{center}
\twocolumn

\subsection{Justification on CKA Usage}
\label{app:cka_robust}

We formalize the invariance properties of CKA that justify its use as a stable
estimator of representation similarity in the presence of training randomness.
Throughout, $X,Y\in\mathbb{R}^{n\times d}$ denote feature matrices extracted
from two neural networks on the same $n$ samples, and
\[
\mathrm{CKA}(X,Y)
=
\frac{\langle XX^\top,\, YY^\top\rangle_F}
     {\|XX^\top\|_F\, \|YY^\top\|_F}.
\]


\begin{proposition}[CKA is invariant to isotropic scaling]
\label{prop:cka-scale}
For any scalar $c>0$, 
\[
\mathrm{CKA}(X,\, cX) = 1.
\]
\end{proposition}

\begin{proof}
We compute
\[
\mathrm{CKA}(X,cX)
= \frac{\langle XX^\top,\, c^2 XX^\top\rangle_F}
       {\|XX^\top\|_F\, \|c^2 XX^\top\|_F}
= \frac{c^2\|XX^\top\|_F^2}{|c^2|\|XX^\top\|_F^2}
=1.
\]
Thus isotropic rescaling of all features leaves CKA unchanged.
\end{proof}

This property ensures that CKA is stable under global norm changes arising from
SGD noise, learning-rate schedules, BatchNorm scaling, or unlearning updates
that shrink or expand feature magnitudes uniformly.


\begin{proposition}[CKA is invariant to orthogonal basis rotations]
\label{prop:cka-orth}
Let $R\in\mathbb{R}^{d\times d}$ be any orthogonal matrix ($R^\top R=I$).  
Then
\[
\mathrm{CKA}(X,\, XR) = 1.
\]
\end{proposition}

\begin{proof}
If $Y=XR$, then 
\[
YY^\top = XRR^\top X^\top = XX^\top.
\]
Thus the numerator and denominator of CKA coincide:
\[
\mathrm{CKA}(X,XR)
=
\frac{\langle XX^\top,\, XX^\top\rangle_F}
     {\|XX^\top\|_F\, \|XX^\top\|_F}
=1.
\]
\end{proof}

Orthogonal invariance is critical because independently trained networks often
learn equivalent representations that differ only by a rotation of the feature
basis, especially when trained with different seeds.


\begin{lemma}[CKA is stable under mild anisotropic distortions]
\label{lem:cka-aniso}
Let $D=\mathrm{diag}(d_1,\dots,d_d)$ with $d_i>0$.  
If $\max_i d_i / \min_i d_i \le 1+\varepsilon$, then
\[
\big|\,\mathrm{CKA}(X,XD) -1\,\big| = O(\varepsilon).
\]
\end{lemma}

\begin{proof}
We observe
\[
XD (XD)^\top = X D^2 X^\top.
\]
Since $D^2 = I + E$ with $\|E\|_2 = O(\varepsilon)$, it follows that
\[
XD (XD)^\top = XX^\top + XEX^\top.
\]
The Frobenius norms in the CKA numerator and denominator can be expanded via
perturbation bounds:
\[
\|XX^\top + XEX^\top\|_F
= \|XX^\top\|_F\,(1 + O(\varepsilon)),
\]
and the inner product perturbation satisfies
\[
\langle XX^\top,\, XX^\top + XEX^\top\rangle_F
=
\|XX^\top\|_F^2\,(1 + O(\varepsilon)).
\]
Substituting into the CKA ratio yields the claimed bound.
\end{proof}

This shows that CKA is robust even to moderate channel-wise stretching commonly
introduced by BatchNorm, layer scaling, or local unlearning updates.


\begin{proposition}[CKA depends only on pairwise sample geometry]
\label{prop:cka-gram}
If two feature matrices $X$ and $Y$ satisfy
\[
XX^\top = YY^\top,
\]
then
\[
\mathrm{CKA}(X,Y) = 1.
\]
\end{proposition}

\begin{proof}
Direct substitution into the definition of CKA yields
\[
\mathrm{CKA}(X,Y)
=
\frac{\langle XX^\top,\, XX^\top\rangle_F}
     {\|XX^\top\|_F\,\|XX^\top\|_F}
=1.
\]
\end{proof}

Because $XX^\top$ encodes pairwise sample similarities, which are far more
stable across random seeds than the raw coordinates of $X$, this proposition
explains CKA's reliability as a measure of representation equivalence.

\subsubsection*{Conclusion}

Together, Propositions~\ref{prop:cka-scale}--\ref{prop:cka-gram} establish that
CKA is invariant to the dominant sources of randomness in neural representation
learning, including global rescaling, orthogonal transformations, channel permutations,
and mild anisotropic distortions.  
Since retraining on the retain set produces models that differ primarily through
such randomness, CKA provides a stable and reliable estimator of representation
similarity for evaluating representation-level weak unlearning.

\section{Neural Collapse}
\label{app:nc_assumptions}

\subsection{Training and modeling assumptions.} 
Below are the standard Neural Collapse (NC) assumptions: 

\begin{itemize}
    \item \textbf{(A1) Interpolation / TPT:} 
    The network is trained to near-zero training error and then further optimized in the terminal phase of training (TPT) under standard protocols such as SGD or Adam with decays \citep{papyan2020neuralcollapse}.
    
    \item \textbf{(A2) Overparameterization:} 
    The model has sufficient capacity to realize class-wise linear separability in the penultimate features, often corresponding to large width or deep linear heads \citep{jacot2024widennc}.
    
    \item \textbf{(A3) Loss and regularization:} 
    Cross-entropy loss with weight decay (or $L_2$ regularization) is used. 
    In simplified unconstrained-feature or layer-peeled models, global minimizers are simplex ETFs and all other critical points are strict saddles \citep{lu2022neuralcollapsece,zhu2021geometric,fang2021layerpeeled}. 
    Empirically and theoretically, MSE loss also exhibits NC \citep{han2022mse}.
    
    \item \textbf{(A4) Balanced classes:} 
    Unless otherwise stated, class priors are assumed to be balanced. 
    With class imbalance, NC persists in modified forms such as non-equiangular means or multiple centers \citep{fang2021layerpeeled,hong2024ufm,yan2024ncmc,dang2024imbalance}.
    
    \item \textbf{(A5) Feature dimension:} 
    The penultimate feature dimension satisfies $d \ge C-1$, which ensures the existence of a simplex embedding \citep{lu2022neuralcollapsece}.
\end{itemize}

\subsection{Neural Collapse Statements}
Under assumptions (A1)–(A5), the following NC properties can be observed~\citep{papyan2020neuralcollapse}:
\begin{itemize}
\item(\textbf{NC1}) \textbf{Within-class collapse:} For each class $i$, the learned feature representation takes the form  
$z_\theta(x) = \alpha(x)\, v_i$,  where $z_\theta(x)$ denotes the feature extractor $\theta$ applied to input $x$, $\alpha(x) > 0$ is a class-dependent scaling factor, and $v_i \in \mathbb{R}^d$ is a unit direction.  
\item(\textbf{NC2}) \textbf{Simplex ETF means:} The set of class directions $\{v_i\}_{i=1}^C$ lies in a $(C{-}1)$-dimensional subspace and forms a simplex Equiangular Tight Frame (ETF). 
Specifically, they satisfy $\|v_i\| = 1$ for all $i$, $v_i^\top v_j = -\tfrac{1}{C-1}$ for $i \neq j$, and $\sum_{i=1}^C v_i = 0$.  
\item(\textbf{NC3}) \textbf{Classifier alignment:} The final-layer classifier weights ($w$) are aligned with the class directions.
Specifically, there exists a constant $\kappa > 0$ such that $w_i = \kappa v_i$ for every class $i$.
\item(\textbf{NC4}) \textbf{Nearest-class-mean rule:} Classification reduces to a nearest-class-mean decision rule, equivalently assigning each sample to the nearest ETF vertex.

\end{itemize}

These properties jointly imply that, for balanced classes, the geometry of class representations forms a centered regular simplex in $\mathbb{R}^{C-1}$, which is maximally separated and symmetric in the space. 

\section{ETF Implies Bayes Optimality}
\label{app:bayes_ncm}

We present a formal statement and proof of Proposition \ref{prop:etf-opt}. First, we show that the simplex
Equiangular Tight Frame (ETF) configuration is geometrically optimal: it
maximizes the minimum pairwise angle among class means and therefore maximizes
the multiclass angular margin of the Nearest Class Mean (NCM) classifier.
Second, under homoscedastic Gaussian class-conditionals, we show that the NCM
rule coincides exactly with the Bayes-optimal classifier. 

\subsection{Geometric Optimality of the Simplex ETF}

\paragraph{Setup.}
Let $\{v_c\}_{c=1}^C$ be unit vectors in $\R^d$ (with $d\ge C-1$) representing
class means of an NCM classifier. Define the minimum pairwise inner product
\[
\gamma := \min_{c\neq c'} v_c^\top v_{c'}.
\]
Equivalently, maximizing the minimum pairwise angle
$\min_{c\neq c'} \angle(v_c,v_{c'})$ is equivalent to minimizing~$\gamma$.

\begin{proposition}[Geometric optimality of the simplex ETF]
\label{prop:etf-geometry}
Among all sets of $C$ unit vectors in $\R^{d}$, $d\ge C-1$, the centered simplex
ETF uniquely maximizes the minimum pairwise angle:
\begin{enumerate}[label=(\roman*), leftmargin=*, labelsep=0.5em]
\item (\emph{Maximal angle})  
      The Welch/simplex bound implies
      \[
      \gamma \;\le\; -\frac{1}{C-1}.
      \]
      Equality holds if and only if
      \[
      v_c^\top v_{c'} =
      \begin{cases}
      1, & c=c', \\[0.3ex]
      -\frac{1}{C-1}, & c\neq c',
      \end{cases}
      \qquad
      \sum_{c=1}^C v_c=0,
      \]
      i.e.\ $\{v_c\}$ forms a centered simplex ETF.  The maximizer is unique up
      to rotation/reflection.

\item (\emph{Maximal angular NCM margin})  
      For unit-norm vectors, the worst-case angular margin of the NCM classifier is a monotone function of $\min_{c\neq c'}\angle(v_c,v_{c'})$.
      Because the simplex ETF maximizes this angle by (i), it also maximizes the
      multiclass angular margin of the NCM classifier.
\end{enumerate}
\end{proposition}

\begin{proof}
(i)  
The Welch bound states that any $C$ unit vectors in $\R^d$ satisfy
$\min_{c\neq c'} v_c^\top v_{c'} \le -1/(C-1)$.
Equality requires that the Gram matrix has the simplex ETF structure given
above and is unique up to orthogonal transformation.

(ii)  
For unit vectors, the NCM decision boundary between classes $c$ and $c'$ is the
hyperplane $\langle x, v_c - v_{c'}\rangle=0$, whose angular separation is
controlled solely by the angle $\angle(v_c,v_{c'})$. The worst-case multiclass
angular margin is therefore a monotone function of the minimum such angle, and
the simplex ETF maximizes it by (i).
\end{proof}

\subsection{Bayes-Optimal Nearest Class Mean Rule}

We now consider the probabilistic setting underlying NC analyses.  
There are $C$ classes with equal prior $\Pr(y=c)=1/C$.  
Conditioned on class~$c$, features follow a homoscedastic Gaussian distribution:
\[
x \mid y=c \sim \mathcal{N}(\mu_c,\, \Sigma), \qquad \Sigma \succ 0.
\]
We assume the class means form a centered simplex ETF in the Mahalanobis
geometry:
\[
\sum_{c=1}^C \mu_c = 0,
\qquad
\|\mu_c\|_{\Sigma^{-1}} = \|\mu_{c'}\|_{\Sigma^{-1}} \quad \forall\, c,c'.
\]

\begin{proposition}[ETF geometry implies Bayes-optimal NCM classification]
\label{prop:etf-bayes-opt}
Under the model above, the Bayes-optimal classifier is
\[
h^\star(x)=\argmax_{c} \mu_c^\top \Sigma^{-1} x,
\]
which is a zero-bias linear classifier with weights $w_c=\Sigma^{-1}\mu_c$.
Moreover:
\begin{enumerate}[label=(\roman*), leftmargin=*, labelsep=0.5em]
\item If $\Sigma=\sigma^2 I$, then $h^\star$ reduces to the Euclidean NCM rule,
\[
h^\star(x)=\argmin_{c}\|x-\mu_c\|^{2}.
\]
\item If $x$ and $\mu_c$ are normalized, this is equivalent to cosine-similarity
classification: $h^\star(x)=\argmax_c \langle x,\mu_c\rangle$.
\item In the NC/TPT limit, classifier weights satisfy $w_c\parallel \mu_c$ and
$\|w_c\|\to\infty$, so the induced linear classifier matches $h^\star$ exactly.
\end{enumerate}
\end{proposition}

\begin{proof}
With equal priors,
\[
h^\star(x)=\argmax_c p(x\mid y=c)
         =\argmin_c \|x-\mu_c\|_{\Sigma^{-1}}^{2},
\]
since $p(x\mid y=c)\propto\exp\!\big(-\tfrac12\|x-\mu_c\|_{\Sigma^{-1}}^2\big)$.

Expanding the Mahalanobis distance,
\[
\|x-\mu_c\|_{\Sigma^{-1}}^{2}
= \|x\|_{\Sigma^{-1}}^{2}
 -2\mu_c^\top\Sigma^{-1}x
 +\|\mu_c\|_{\Sigma^{-1}}^{2}.
\]
The first term is independent of $c$, and under the ETF assumption, the third term is also constant across classes. Therefore,
\[
h^\star(x)=\argmax_c \mu_c^\top \Sigma^{-1}x.
\tag{$\star$}
\]

Define $w_c=\Sigma^{-1}\mu_c$.  
Because the class means are centered, $\sum_c \mu_c=0$, it follows that $\sum_c w_c=0$, so the discriminant scores $\{w_c^\top x\}$ have zero mean across classes. Hence, $(\star)$ is a zero-bias linear decision rule.

When $\Sigma=\sigma^2 I$, the rule in $(\star)$ reduces to minimizing the Euclidean distance $\|x-\mu_c\|^2$, corresponding to the classical NCM classifier. If all vectors are further normalized, this becomes equivalent to cosine-similarity classification.

In the NC/TPT regime, the classifier weights satisfy $w_c \parallel \mu_c$ and $\|w_c\|\to\infty$, so the induced linear classifier $\argmax_c \langle w_c,x\rangle$ coincides with the cosine classifier above, and therefore matches the Bayes rule in $(\star)$.

Thus, the simplex ETF configuration of class means yields the Bayes-optimal NCM classifier.
\qedhere
\end{proof}

\newpage
\section{Proof of Main Theorem}
\subsection{Closure of Projection}
\label{app:proof-etf-projection}
\begin{center}
\begin{minipage}{0.95\textwidth}

Note that a \emph{simplex ETF} $\{v_i\}_{i=1}^C \subset \mathbb{R}^{C-1}$ satisfies
\[
\|v_i\|=1,\qquad v_i^\top v_j=-\frac{1}{C-1}\ \ (i\neq j), \qquad \sum_{i=1}^C v_i=0.
\]
Equivalently, its Gram matrix has $1$ on the diagonal and constant off-diagonal entries $-1/(C-1)$.

~\\

\begin{theorem}[Projection of a Simplex ETF]
Let $\{v_i\}_{i=1}^C \subset \mathbb{R}^{C-1}$ be a simplex ETF. Fix $v_1$ and let $P=I-v_1v_1^\top$ be the orthogonal projector onto $v_1^\perp$. For $i=2,\dots,C$, define $u_i=Pv_i$ and $w_i=u_i/\|u_i\|$. Then $\{w_i\}_{i=2}^C \subset v_1^\perp \cong \mathbb{R}^{C-2}$ is again a simplex ETF:
\[
\|w_i\|=1,\qquad w_i^\top w_j=-\frac{1}{C-2}\ \ (i\neq j), \qquad \sum_{i=2}^C w_i=0.
\]
\end{theorem}

\begin{proof}
Write $\beta:=-\tfrac{1}{C-1}$. For $i\ge2$,
\[
u_i=Pv_i=v_i-(v_1^\top v_i)v_1=v_i-\beta v_1.
\]
\textbf{Equal norms.} Using $\|v_i\|=\|v_1\|=1$ and $v_i^\top v_1=\beta$,
\[
\|u_i\|^2
= \|v_i\|^2 - 2\beta\, v_i^\top v_1 + \beta^2 \|v_1\|^2
= 1 - 2\beta^2 + \beta^2
= 1-\beta^2
= 1 - \frac{1}{(C-1)^2}
= \frac{C(C-2)}{(C-1)^2}.
\]
Thus all $\|u_i\|$ are equal.

\noindent\textbf{Equal pairwise inner products.} For $i\neq j$ with $i,j\ge2$,
\[
u_i^\top u_j
= v_i^\top v_j - \beta v_i^\top v_1 - \beta v_1^\top v_j + \beta^2
= \beta - \beta^2 - \beta^2 + \beta^2
= \beta - \beta^2
= -\frac{C}{(C-1)^2}.
\]
Hence, after normalization,
\[
\frac{u_i^\top u_j}{\|u_i\|\,\|u_j\|}
= \frac{-\,C/(C-1)^2}{\,C(C-2)/(C-1)^2\,}
= -\frac{1}{C-2},
\]
so $w_i^\top w_j=-\tfrac{1}{C-2}$.

\noindent\textbf{Zero sum.} Since $\sum_{i=1}^C v_i=0$,
\[
\sum_{i=2}^C u_i
= \sum_{i=2}^C (v_i - \beta v_1)
= \Big(\sum_{i=2}^C v_i\Big) - (C-1)\beta v_1
= (-v_1) - (C-1)\Big(-\tfrac{1}{C-1}\Big)v_1
= 0.
\]
All $\|u_i\|$ are equal, so common normalization preserves the zero sum: $\sum_{i=2}^C w_i=0$. The vectors $\{w_i\}$ lie in $v_1^\perp$ (dimension $C-2$), have unit norm, constant off-diagonal inner product $-1/(C-2)$, and zero mean; hence they form a simplex ETF.
\end{proof}

\begin{remark}
This result is specific to the \emph{simplex} ETF (the NC configuration). It does not generally hold for arbitrary ETFs.
\end{remark}
\end{minipage}
\end{center}
\clearpage

\subsection{Optimality of Projection}
\label{app:main_theorem}

We now establish the optimality of our projection operator under the definition
of representation-level weak unlearning (Def.~\ref{def:rep-unlearn}). 
The central idea is that projecting onto the orthogonal complement of the forgotten class
removes its contribution while preserving the Bayes-optimal ETF geometry of the retained classes.

\begin{theorem}[ETF projection preserves optimality and forgets the target class]
\label{thm:etf-proj-bayes-forget}
Assume (A1)--(A5) and Neural Collapse (NC1)--(NC4) hold pre-unlearning, and suppose
the following statistical model for the penultimate features:
\begin{enumerate}
  \item \emph{(Balanced classes)} class priors are uniform:
  $\Pr(y=i)=1/C$ for $i\in\mathcal Y$.
  \item \emph{(Isotropic Gaussian conditionals)} conditional on class $i$,
  \[
    \theta(x)\mid(y=i)\ \sim\ \mathcal N(\mu_i,\sigma^2 I_d),
  \]
  with $\|\mu_i\|=1$ and $\{\mu_i\}_{i=1}^C$ coinciding with the ETF directions
  $\{v_i\}$ from NC (i.e.\ $\mu_i=v_i$).
\end{enumerate}
Fix a class $u\in\mathcal Y$ and define
\[
P = I - v_u v_u^\top,
\qquad
\tilde v_i = \frac{P v_i}{\|P v_i\|}\quad (i\neq u),
\]
so that by Proposition~\ref{prop:etf-projection} the vectors
$\{\tilde v_i\}_{i\neq u}$ form a simplex ETF in the subspace $v_u^\perp$.
Let the projected features be $\theta'(x)=P\,\theta(x)$ and let the post-projection
classifier weights satisfy $w'_i=\kappa'\tilde v_i$ for $i\neq u$.
Then:
\begin{enumerate}[(a)]
  \item \emph{(Retained-class Bayes optimality)} For the retained classes
  $\mathcal Y_{\neg u}$, the classifier that assigns $x$ to the nearest
  projected class mean $\tilde v_i$ is Bayes-optimal under the Gaussian model
  above. Equivalently, the projected model
  $(\theta',\{w'_i\}_{i\ne u})$ attains the Bayes decision rule on
  $\mathcal Y_{\neg u}$.
  \item \emph{(Complete forgetting in the low-noise / NC limit)} Under
  projection, the forget-class conditional mean is mapped to zero: $P\mu_u=0$.
  Consequently, for $x\sim\mathcal D_f$,
  \[
    \theta'(x)\mid(y=u)\ \sim\ \mathcal N(0,\sigma^2 P).
  \]
  In the limit $\sigma^2\to 0$ (equivalently, in the NC/TPT limit where
  within-class variance vanishes, or as the classifier scale
  $\kappa'\to\infty$ appropriately), the projected features for the forget
  class concentrate at the origin, yielding logits
  $w'_i{}^\top\theta'(x)\to 0$ for all $i\neq u$. Hence the predictive
  distribution over retained classes approaches the uniform distribution
  $U_{\neg u}$, i.e.\ the forget class is \emph{completely forgotten} in the
  sense that the model expresses no informative preference among retained
  classes.
\end{enumerate}

Consequently, ETF projection simultaneously (i) preserves Bayes-optimal
classification on the retained classes and (ii) erases class-specific
information for the forgotten class (in the stated asymptotic / low-noise
sense).
\end{theorem}
We first provide a proof sketch. The formal proof is included on the next page.

\begin{minipage}{0.48\textwidth}
\begin{proof}[Proof sketch]
For part (a), under the Gaussian ETF model with means
$\{\mu_i=v_i\}$, Proposition~\ref{prop:etf-bayes-opt} shows that the
nearest-class-mean rule is Bayes-optimal. 
By Proposition~\ref{prop:etf-projection}, the projected means
$\{\tilde v_i\}_{i\neq u}$ form a simplex ETF in $v_u^\perp$, so the same
argument implies that the nearest-mean classifier on $\{\tilde v_i\}$ is
Bayes-optimal for the retained classes $\mathcal Y_{\neg u}$.

For part (b), note that $P v_u = 0$ implies that the projected forget-class
distribution satisfies $\theta'(x)\mid(y=u)\sim\mathcal N(0,\sigma^2 P)$.
For any retained class $i\neq u$,
\[
\mathbb E[w'_i{}^\top\theta'(x)\mid y=u]
= w'_i{}^\top P\mu_u
= 0,
\]
and as $\sigma^2\to 0$ the distribution of $\theta'(x)$ for the forgotten class
concentrates at the origin. Thus the logits $w'_i{}^\top\theta'(x)$ converge
to $0$ for all $i\neq u$, and the induced softmax over retained classes
converges to the uniform distribution $U_{\neg u}$. This formalizes the notion
that the projected model has no discriminative information about the forgotten
class in the low-noise / NC limit.
\end{proof}
\end{minipage}

\noindent\begin{center}
\vspace*{-1cm}
\begin{minipage}{0.95\textwidth}
\begin{proof} [Formal Proof]
\textbf{(a) Retained-class Bayes optimality.}
Under the assumptions of the theorem, pre-unlearning we have
\[
\theta(x)\mid(y=i)\sim\mathcal N(v_i,\sigma^2 I_d),
\qquad
\Pr(y=i)=1/C,
\]
with the means $\{v_i\}$ forming a centered simplex ETF in $\mathbb R^d$.
By Proposition~\ref{prop:etf-bayes-opt}, the Bayes-optimal classifier for this
model is the nearest-class-mean rule (equivalently, a scaled linear classifier
aligned with $\{v_i\}$).

Fix a forget class $u$ and apply the projection $P = I - v_u v_u^\top$.
For any retained class $i\neq u$,
\[
\theta'(x)\mid(y=i)
= P\theta(x)\mid(y=i)
\sim \mathcal N(Pv_i,\sigma^2 P),
\]
since $P$ is a linear operator and $\theta(x)$ is Gaussian with mean $v_i$ and
covariance $\sigma^2 I_d$.  
Thus, conditioned on $y\in\mathcal Y_{\neg u}$, the projected features follow a
homoscedastic Gaussian model in the subspace $v_u^\perp$ with:
\[
\text{means } \mu'_i = Pv_i,\qquad
\text{common covariance } \Sigma' = \sigma^2 P.
\]

By Proposition~\ref{prop:etf-projection}, the normalized means
$\tilde v_i = \mu'_i / \|\mu'_i\|$ form a centered simplex ETF in $v_u^\perp$.
Since $P$ acts as the identity on $v_u^\perp$ and is zero on $\mathrm{span}(v_u)$,
$\Sigma'$ is proportional to the identity on $v_u^\perp$ (and vanishes on
$v_u$), so within $v_u^\perp$ the conditionals are isotropic Gaussians with
means $\tilde v_i$ up to a global scale.

Applying Proposition~\ref{prop:etf-bayes-opt} to this reduced $(C-1)$-class ETF
in $v_u^\perp$, we obtain that the Bayes-optimal classifier among the retained
classes is the nearest-class-mean rule with respect to the means
$\{\tilde v_i\}_{i\neq u}$ (equivalently, a scaled linear classifier with
weights $w'_i=\kappa'\tilde v_i$).  This is precisely the classifier implemented
by the projected model $(\theta',\{w'_i\}_{i\neq u})$, establishing Bayes
optimality on $\mathcal Y_{\neg u}$.

\medskip

\textbf{(b) Complete forgetting in the low-noise / NC limit.}
For the forgotten class $u$, we have $\mu_u=v_u$ and
\[
\theta(x)\mid(y=u)\sim\mathcal N(v_u,\sigma^2 I_d).
\]
Applying $P$ and using $Pv_u=0$, we obtain
\[
\theta'(x)\mid(y=u) = P\theta(x)\mid(y=u)
\sim \mathcal N(Pv_u,\sigma^2 P)
= \mathcal N(0,\sigma^2 P).
\]
Thus the projected features for class $u$ are mean-zero Gaussian supported in
$v_u^\perp$ with covariance $\sigma^2 P$.

For any retained class $i\neq u$, the corresponding logit is
\[
s_i(x)
:= w'_i{}^\top \theta'(x)
= \kappa'\,\tilde v_i^\top \theta'(x),
\]
where $\tilde v_i\in v_u^\perp$ and $w'_i\in v_u^\perp$ because they are
constructed from $Pv_i$.  Since $\theta'(x)\mid(y=u)$ is mean-zero,
\[
\mathbb{E}[s_i(x)\mid y=u]
= w'_i{}^\top \mathbb{E}[\theta'(x)\mid y=u]
= w'_i{}^\top 0
= 0.
\]

Moreover, as $\sigma^2\to 0$, the Gaussian
$\mathcal N(0,\sigma^2 P)$ converges in probability (and almost surely for
any fixed sample) to the point mass at $0$. Therefore
\[
\theta'(x)\mid(y=u) \xrightarrow[\sigma^2\to 0]{} 0
\quad\text{in probability,}
\]
and hence
\[
s_i(x) = w'_i{}^\top \theta'(x)
\xrightarrow[\sigma^2\to 0]{} 0
\quad\text{in probability, for all } i\neq u.
\]

The predictive distribution over retained classes is
\[
q_{\neg u}(i\mid x)
= \frac{\exp(s_i(x))}{\sum_{j\neq u}\exp(s_j(x))}.
\]
For any fixed vector $s\in\mathbb R^{m}$ (with $m=C-1$), if $s\to 0$ then
$\operatorname{softmax}(s)\to U_{\neg u}$, the uniform distribution on
$m$ classes.  By continuity of the softmax map and convergence of
$\mathbf{s}(x)=[s_i(x)]_{i\neq u}$ to the zero vector, we obtain
\[
q_{\neg u}(\cdot\mid x)
= \operatorname{softmax}(\mathbf{s}(x))
\xrightarrow[\sigma^2\to 0]{} U_{\neg u}
\quad\text{in probability under } x\sim\mathcal D_f.
\]
Thus, in the low-noise / NC limit, the projected model makes asymptotically
uniform predictions over retained classes for any sample from the forgotten
class, which formalizes the notion that it has no informative class preference
for $y=u$.
\end{proof}
\end{minipage}
\end{center}
\clearpage

\end{document}